\documentclass{article}

\usepackage{algorithm}
\usepackage{adjustbox}
\usepackage{graphicx}
\usepackage[T1]{fontenc}
\usepackage[utf8]{inputenc}
\usepackage{color}
\usepackage{microtype}
\usepackage{natbib}
\usepackage[normalem]{ulem}
\usepackage{wrapfig}
\usepackage{xcolor}

\usepackage{booktabs} %

\usepackage{geometry}
\usepackage{hyperref}
 \definecolor{mydarkblue}{rgb}{0,0.08,0.45}
  \hypersetup{ %
    pdfkeywords={},
    pdfborder=0 0 0,
    pdfpagemode=UseNone,
    colorlinks=true,
    linkcolor=mydarkblue,
    citecolor=mydarkblue,
    filecolor=mydarkblue,
    urlcolor=mydarkblue,
    }

\usepackage{amsmath}
\usepackage{amssymb}
\usepackage{mathtools}
\usepackage{amsthm}

\allowdisplaybreaks
\makeatletter

\theoremstyle{plain}

\theoremstyle{definition}

\theoremstyle{remark}

\usepackage{amsfonts}%
\usepackage{booktabs}%
\usepackage{nicefrac}%

\usepackage{url}%
\usepackage[inline]{enumitem}

\usepackage[mathscr]{eucal}
\usepackage{booktabs,colortbl}
\usepackage{subcaption}
\usepackage{verbatim}

\renewcommand{\Pr}{\mathbb{P}}
\newcommand{\defEq}{\stackrel{.}{=}}

\global\long\def\calF{\mathscr{F}}%
\global\long\def\calX{\mathscr{X}}%
\global\long\def\calY{\mathscr{Y}}%
\global\long\def\bbI{\mathbb{I}}%
\global\long\def\bbP{\mathbb{P}}%
\global\long\def\bbR{\mathbb{R}}%
\global\long\def\lossce{\ell_{\mathrm{CE}}}%
\newcommand{\Real}{\mathbb{R}}

\newcommand{\xPos}{x^+}
\newcommand{\xNeg}{x^-}

\newcommand{\tick}{$\checkmark$}
\newcommand{\cross}{$\times$}

\newcommand{\fade}[1]{{\color{gray!90}#1}}

\providecommand{\propositionname}{Proposition}
\providecommand{\theoremname}{Theorem}

\theoremstyle{plain}
\newtheorem{thm}{\theoremname}
\theoremstyle{plain}
\newtheorem{prop}[thm]{\propositionname}
\ifx\proof\undefined
\newenvironment{proof}[1][\proofname]{\par
	\normalfont\topsep6\p@\@plus6\p@\relax
	\trivlist
	\itemindent\parindent
	\item[\hskip\labelsep\scshape #1]\ignorespaces
}{%
	\endtrivlist\@endpefalse
}
\providecommand{\proofname}{Proof}
\fi

\newtheorem{lem}[thm]{\protect\lemmaname}
\providecommand{\lemmaname}{Lemma}

\newcommand{\ourtitle}{ELM: Embedding and Logit Margins for Long-Tail Learning}

\author{
Wittawat Jitkrittum, \:
Aditya Krishna Menon,\:
Ankit Singh Rawat,\:
Sanjiv Kumar\\[1mm]
\texttt{\{wittawat, adityakmenon, ankitsrawat, sanjivk\}@google.com} \\[4mm]
Google, New York, USA
}

\title{\ourtitle}
\begin{document}

\maketitle
\begin{abstract}

Long-tail learning is the problem of learning under skewed label distributions, which pose a challenge for standard learners.  Several recent approaches for the problem have proposed enforcing a suitable margin in logit space.  Such techniques are intuitive analogues of the guiding principle behind SVMs, and are equally applicable to linear models and neural models.  However, when applied to neural models, such techniques do not explicitly control the geometry of the learned embeddings.  This can be potentially sub-optimal, since embeddings for tail classes may be diffuse, resulting in poor generalization for these classes.  We present Embedding and Logit Margins (ELM), a unified approach to enforce margins in logit space, and regularize the distribution of embeddings.  This connects losses for long-tail learning to proposals in the literature on metric embedding, and contrastive learning.  We theoretically show that minimising the proposed ELM objective helps reduce the generalisation gap. The ELM method is shown to perform well empirically, and results in tighter tail class embeddings.

\end{abstract}

\section{Introduction}
Practical classification problems often possess skewed label distributions,
which pose a challenge for standard learners.
This problem of learning under \emph{class imbalance}~\citep{Kubat:1997a,Chawla:2002,HeGa09},
or \emph{long-tail learning},
has received renewed interest in the context of neural models~\citep{VanHorn:2017,Buda:2017,Liu:2019}.
Successful approaches to the problem
include modifying the training data (e.g., by up- or down-sampling different labels~\citep{KubatMa97,Chawla:2002,Wallace:2011,Mikolov:2013,Mahajan:2018,Yin:2018,Zhang:2019}),
modifying the classification rule (e.g., by applying varying thresholds for the different classes~\citep{Fawcett:1996,Provost:2000,Maloof03,King:2001,Collell:2016}),
and
modifying the loss function (e.g., by penalising errors on rare labels more strongly~\citep{Zhang:2017,Cui:2019,Cao:2019,Tan:2020,Jamal:2020,Ren:2020,Wu:2020,Menon:2021,SamChe2021,Kini:2021,Wang:2021b}).

Our interest in this paper is in the latter class of loss modification methods.
These have garnered particular interest of late,
with
several recent works~\citep{Cao:2019,Tan:2020,Ren:2020,Menon:2021,Kini:2021,Wang:2021b}
establishing the value of enforcing \emph{asymmetric logit margins}.
Such techniques are intuitive
analogues of the guiding principle behind SVMs,
and aim to clearly separate the scores for rare versus dominant classes.
Despite their success, such techniques are not without limitation.
For example,
when applied to neural
models, they do not explicitly control the distribution
of the learned embeddings themselves.
This can be potentially sub-optimal,
since embeddings for tail classes may be diffuse,
as has been empirically observed~\citep{Zhang:2017,Yin:2018,Liu:2019,Zhong:2019,Ye:2020,SamChe2021,Wang:2021}.

In this paper,
we present ELM,
a framework that enforces both
\emph{Embedding
and Logit Margins}.
In a nutshell, ELM enforces margins
in logit space, and regularize the distribution of embeddings.
This
connects losses for long-tail learning to proposals in the literature
on metric learning~\citep{Weinberger:2009}, and contrastive learning~\citep{Khosla:2020}.
Theoretically, we show
how ELM encourages a better approximation to the Bayes solution, by
ensuring that class-conditionals are more Gaussian.
Empirically, ELM is
shown to perform well, and results in tighter embeddings
(cf.~ Figure~\ref{fig:toy2d_embed}).
In sum, our contributions are:

\begin{enumerate}[label=(\roman*),itemsep=0pt,topsep=0pt,leftmargin=16pt]
    \item{ we propose ELM~\eqref{eqn:elm}, a technique that enforces both embedding and logit margins for long-tail learning,
    leveraging insights from
    metric~\citep{Weinberger:2009} and representation learning~\citep{Wen:2016};}
    \item{we establish the benefits of enforcing embedding and logit margins, by showing that ELM encourages a better approximation to the Bayes-optimal classifier (\S\ref{sec:analysis}); and,}
    \item{ we present experiments on synthetic and real-world datasets that confirm the value of ELM against existing methods (\S\ref{sec:experiments}),
    and in particular demonstrate the import of enforcing margins in both logit and embedding space.}
\end{enumerate}
\begin{figure*}
    \centering
    \resizebox{\linewidth}{!}{
    \subcaptionbox{Training data.}{
    \includegraphics[scale=0.24]{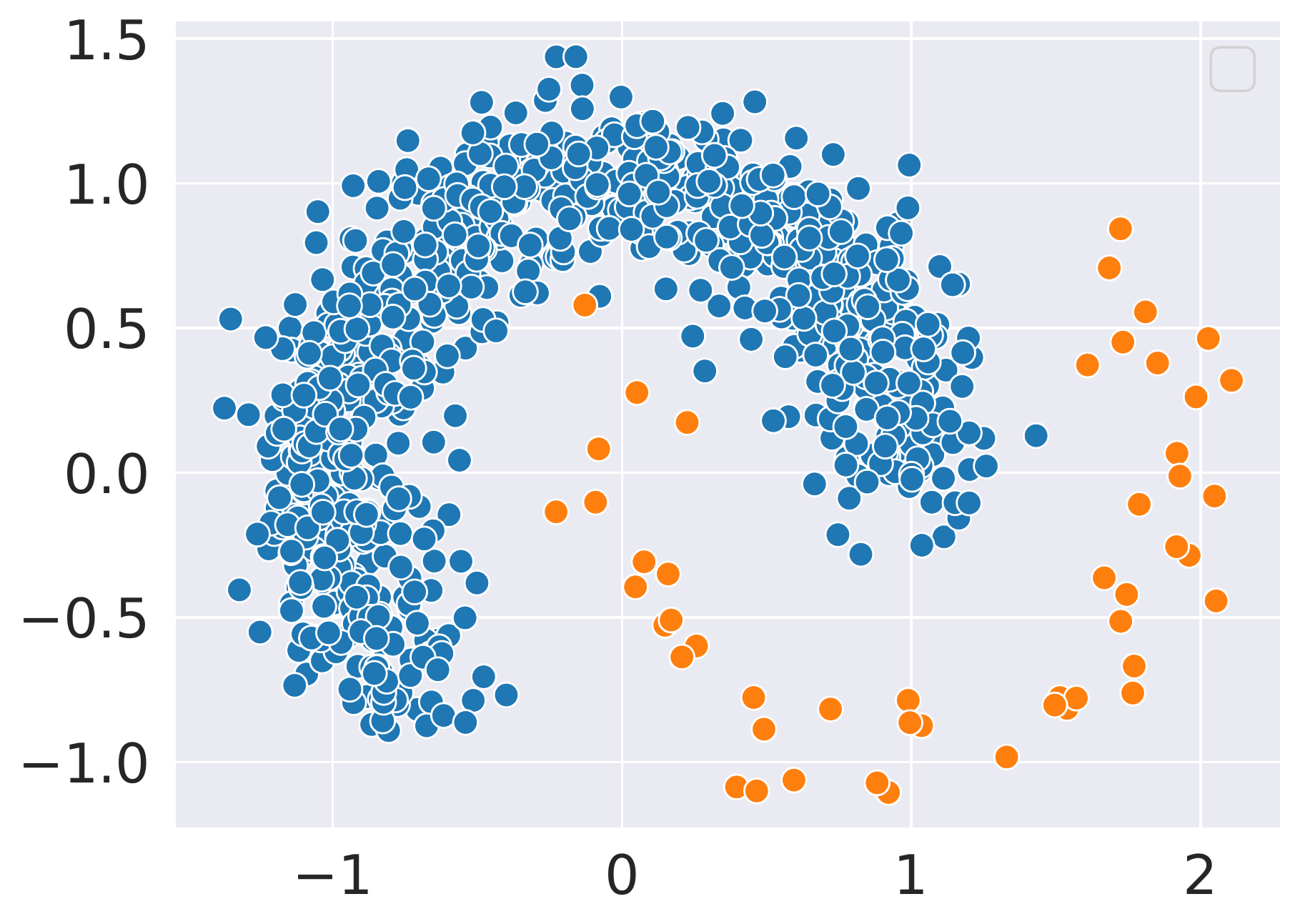}
    }
    \quad
    \subcaptionbox{Cross-entropy with logit margin.}{
    \includegraphics[scale=0.24]{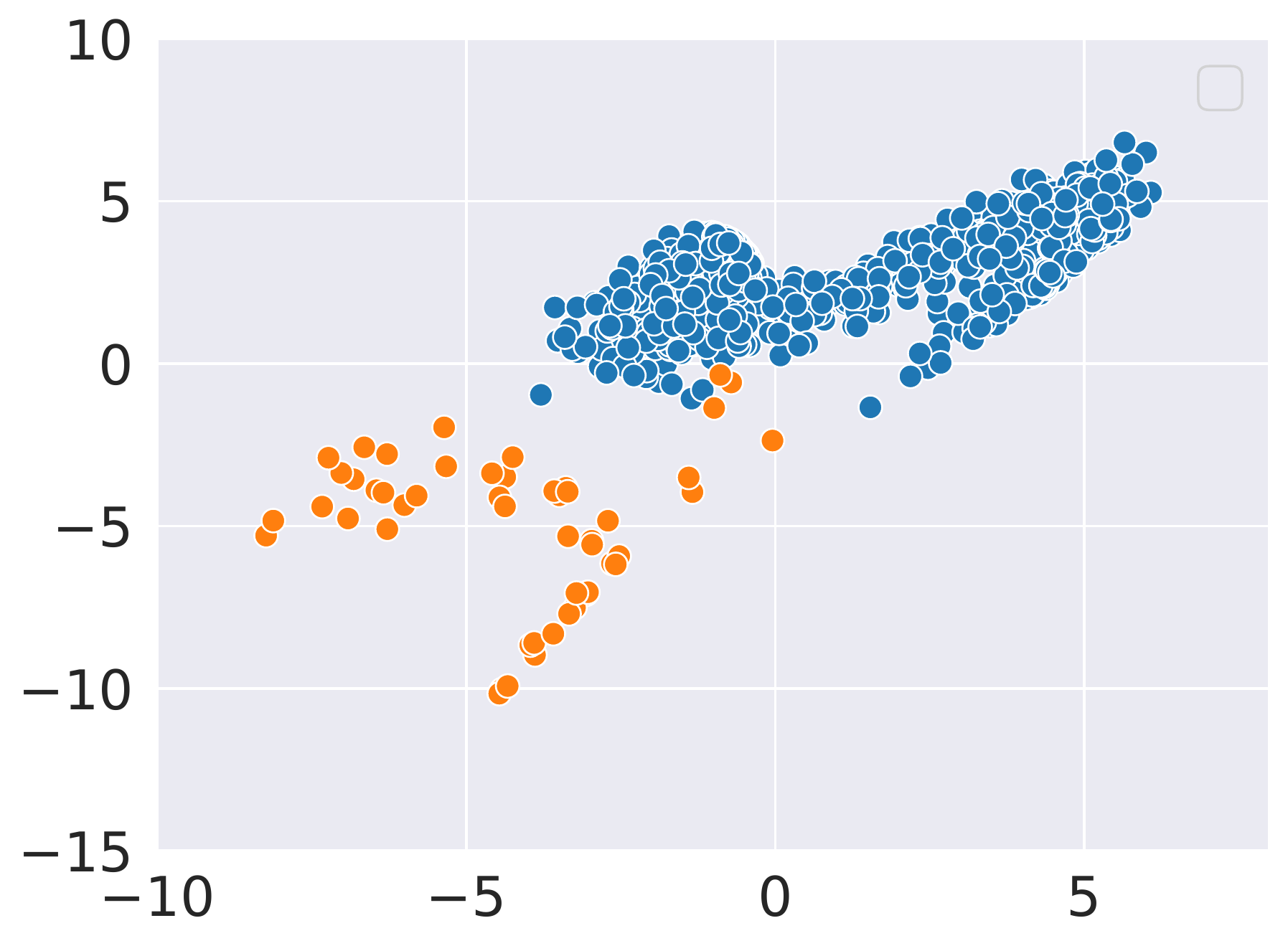}
    }
    \quad
    \subcaptionbox{ELM (proposed).}{
     \includegraphics[scale=0.24]{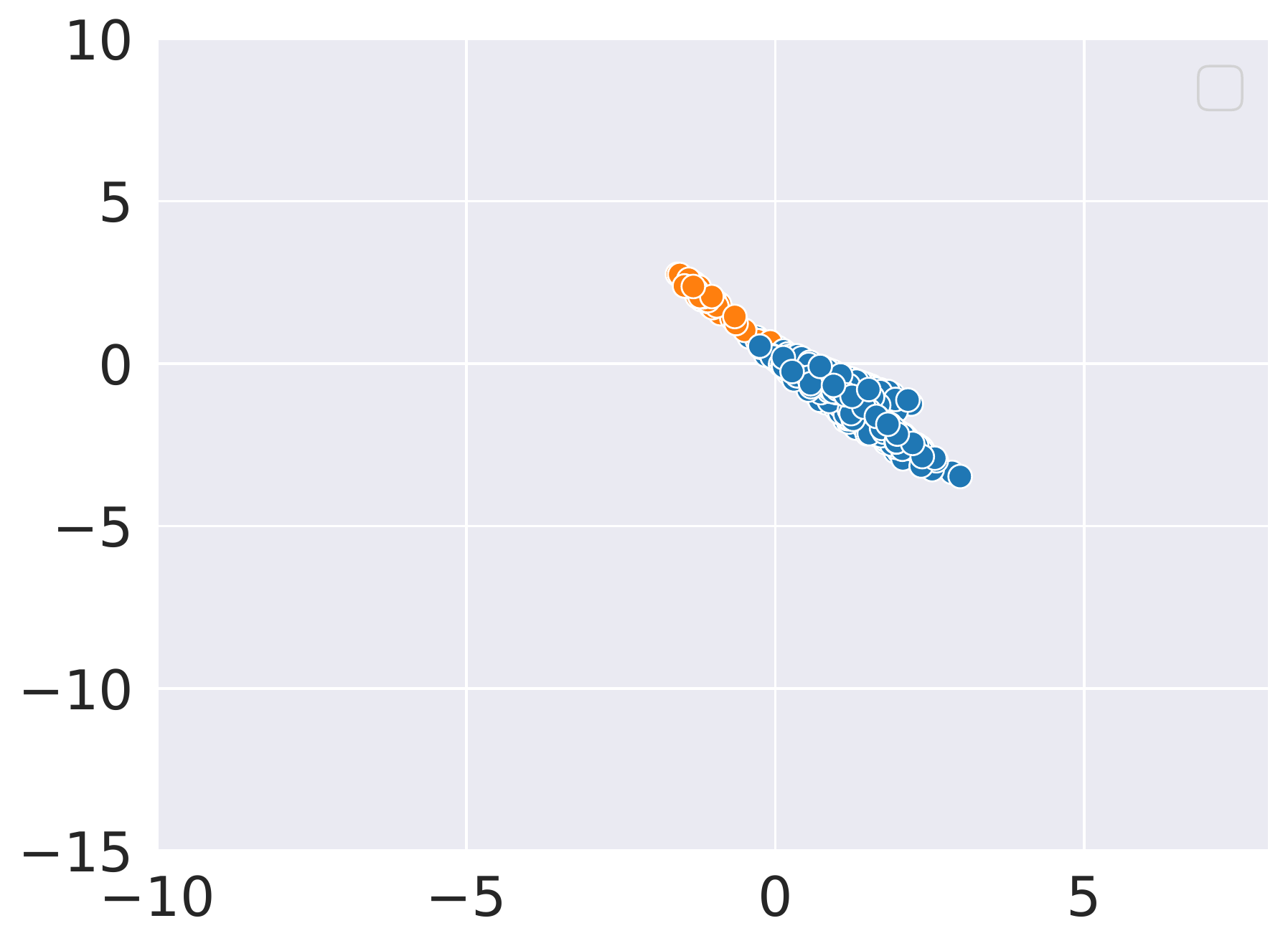}
     \includegraphics[scale=0.24]{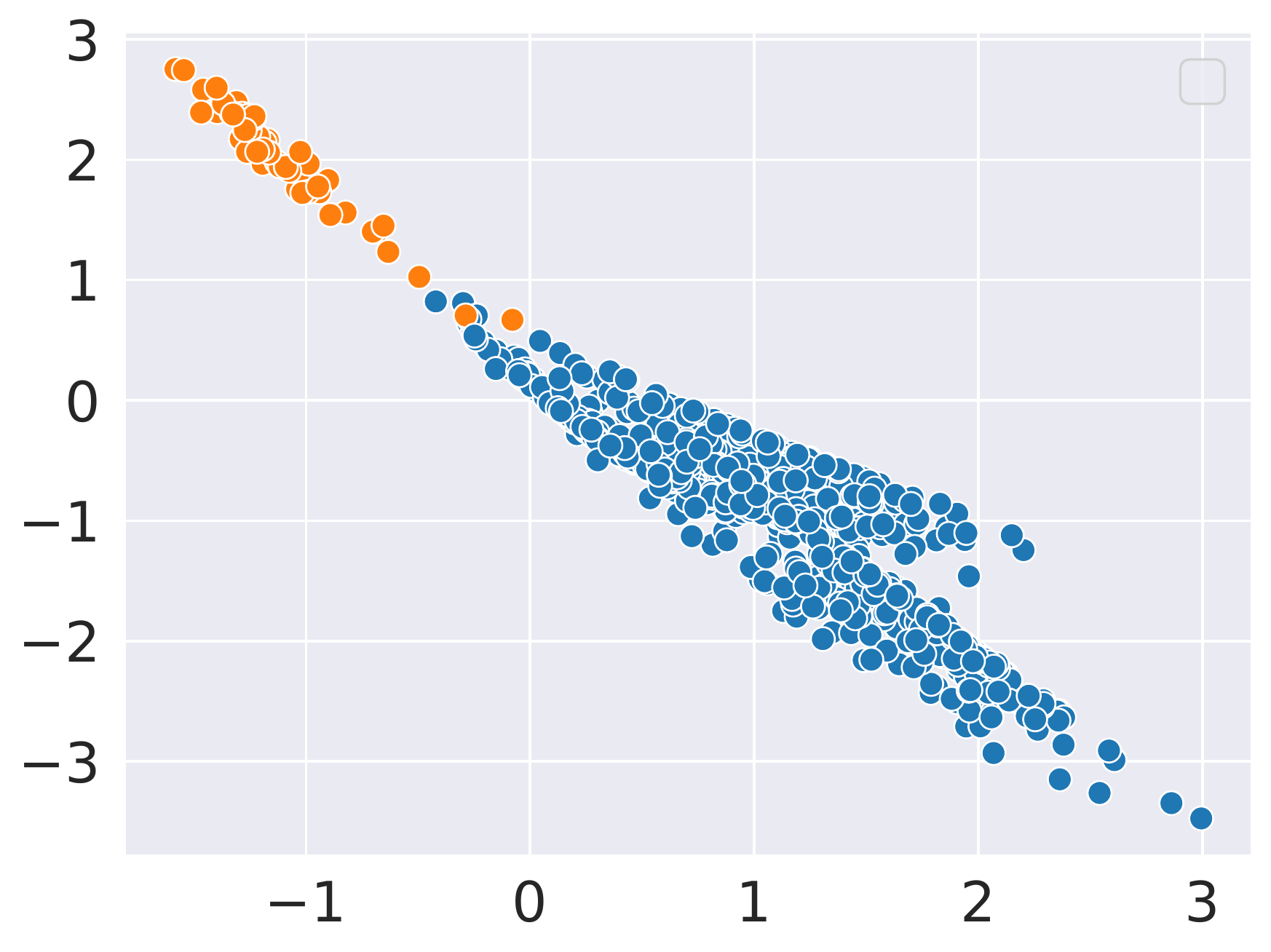}
    }
    }
    \caption{Illustration of embeddings produced from a three-layer ReLU network trained on an imbalanced variant of the two-moon problem.
    The embeddings learned by our proposed ELM regulariser are
    more compact than those from cross-entropy with logit margins.
    In particular, embeddings from the rare class (orange) are pulled
    more tightly together.
    See~\S\ref{sec:limits} for details.
    }
    \label{fig:toy2d_embed}
\end{figure*}

\section{Background and Notation}
\subsection{Multi-Class Classification}

Let $\calX$ be the domain of input instances, and $\calY=[L]\defEq\{1,\ldots,L\}$
be the domain of class labels. Given a training sample $S\defEq\{(x_{i},y_{i})\}_{i=1}^{N}\stackrel{\text{i.i.d.}}{\sim}\bbP^{N}$
where $\bbP$ is a joint distribution defined on $\calX\times\calY$,
the goal of the multi-class classification problem is to learn a
\emph{scorer} $f\colon\calX\to\Real^{L}$ with $f(x)=(f_{1}(x),\ldots,f_{L}(x))^{\top}$
so as to minimize the expected loss $\ell\colon\calY\times\Real^{L}\to\bbR_+$. That is, one
solves the following optimization problem:
\begin{align}
\min_{f\in\calF} & 
\thinspace R( f ) 
\defEq 
\mathbb{E}_{(x,y)\sim\bbP} \thinspace \left[ \ell(y,f(x)) \right],
\label{eq:pop_erm}
\end{align}
where $\calF$ is a class of models for the scorer.
For the zero-one loss
$\ell_{01}(y,f(x))\defEq \bbI[y\in\arg\max_{y'\in\calY} f_{y'}(x)]$,
\eqref{eq:pop_erm}
coincides with the notion of misclassification error.
Since $\ell_{01}$
is not differentiable, a commonly used surrogate loss
is the softmax cross entropy $\lossce(y,f(x)) \defEq -\log\left[\frac{e^{f_{y}}}{\sum_{y'\in\calY}e^{f_{y'}}}\right]=\log\left[1+\sum_{y'\neq y}e^{f_{y'}(x)-f_{y}(x)}\right]$.
Since $f$ is an argument to the softmax function, $f_{1}(x),\ldots,f_{L}(x)$
are also known as the logits for $x$.

\subsection{Long-Tail Learning}

Practical classification problems often posses a skewed label distribution $\Pr( y )$.
This problem of learning under \emph{class imbalance} is a classical area of study~\citep{Kubat:1997a,Chawla:2002,HeGa09},
which
has received renewed interest in the context of neural models
in the area of \emph{long-tail learning}~\citep{VanHorn:2017,Buda:2017,Liu:2019,Johnson:2019}.
The core challenge in such settings is ensuring that rare labels are not systematically misclassified,
owing to their limited representation in the training data.

Formally, this is typically encapsulated as the goal of minimising the \emph{balanced error},
which posits a uniform label distribution $\Pr( y )$ for evaluation:
\begin{align}
\min_{f\in\calF} & \thinspace R_{\rm bal}( f ) \defEq \frac{1}{L} \sum\nolimits_{y \in [L]} \mathbb{E}_{x \mid y}\thinspace \left[ \ell(y,f(x)) \right].\label{eq:pop_bal}
\end{align}
Most
successful approaches 
follow one of three strategies:
\begin{enumerate}[label=(\roman*),itemsep=0pt,topsep=0pt,leftmargin=16pt]
    \item modifying the training data to make it more balanced (e.g., by up- or down-sampling different labels~\citep{KubatMa97,Chawla:2002,Wallace:2011,Mikolov:2013,Xue:2015,Mahajan:2018,Yin:2018,Zhang:2019}),

    \item modifying the classification rule to ensure greater representation of rare classes (e.g., applying per-class thresholds~\citep{Fawcett:1996,Provost:2000,Maloof03,King:2001,Collell:2016,Kang:2020,Zhang:2021}),
and

    \item modifying the loss function to penalise errors on rare labels more strongly (e.g., by introducing appropriate asymmetry~\citep{Zhang:2017,Cui:2019,Cao:2019,Tan:2020,Jamal:2020,Ren:2020,Wu:2020,Menon:2021,Deng:2021,Kini:2021,Wang:2021b}).
\end{enumerate}
The above is not exhaustive, and other strategies have also been pursued~\citep{Yang:2020b,Sahoo:2020,Liu:2019,Liu:2020,Chu:2020,Tang:2020b,SamChe2021,Ye:2021}.
Amongst loss modification techniques, a popular strategy involves augmenting the softmax cross-entropy with \emph{logit margins}.
Specifically, these involve an instantiation of the loss
\begin{equation}
    \label{eqn:xent-margin}
    \ell_{\rm mar}( y, f(x) ) \defEq \log\left[ 1+\sum\nolimits_{y'\neq y} e^{\Delta_{yy'} + f_{y'}(x)-f_{y}(x)} \right],
\end{equation}
where $\Delta_{yy'}$ is some set of margins between labels $y$ and $y'$.
Examples of such $\Delta$
include
$\Delta_{yy'} = \frac{1}{\Pr(y)^{1/4}}$~\citep{Cao:2019},
$\Delta_{yy'} = {\Pr(y')}$~\citep{Tan:2020},
and
$\Delta_{yy'} = \log \frac{\Pr(y')}{\Pr(y)}$~\citep{Ren:2020,Menon:2021,Wang:2021b}.
Intuitively, such a loss can be seen as a soft approximation to
$\max_{y'\neq y} [ \Delta_{yy'} + f_{y'}(x)-f_{y}(x) ]_+$, where $[a]_+
\defEq \max(0, a)$,
and thus encourages a sufficiently large gap between the logits for $y$ and $y'$.
By ensuring that $\Delta_{yy'}$ is large for rare ``positive'' labels $y$ and/or dominant ``negative'' labels $y'$,
one mitigates confusing a rare label for a dominant one.
In the sequel, we shall primarily be interested in the choice $\Delta_{yy'} = \log \frac{\Pr(y')}{\Pr(y)}$, which possesses good empirical performance compared to alternatives~\citep{Ren:2020,Menon:2021,Wang:2021b}.

\subsection{The Limits of Logit Margins}
\label{sec:limits}

While logit margins have enjoyed considerable success,
they alone
may not be enough to guarantee accurate predictions for tail samples.
Consider a synthetic setup, where we have 2D data with binary labels $\{ \mathrm{Head}, \mathrm{Tail} \}$, following a similar setup to the ``two moons'' distribution~\citep{Zhou:2003}.
We set $\Pr( y = \mathrm{Tail} ) = 5\%$, so that the label distribution is imbalanced.
To learn a nonlinear classifier, we use a three layer feedforward network with ReLU activation,
with $\{ 16, 8, 2 \}$ hidden units respectively.

The use of two hidden units for the pre-output layer facilitates ready visualization.
Figure~\ref{fig:toy2d_embed} illustrates the learned embeddings under
minimisation of the cross-entropy with the logit-adjusted margin~\eqref{eqn:xent-margin}, which achieves near perfect test accuracy.
Despite their good performance, we see that the learned embeddings are diffuse.
By contrast, the embeddings for each class become relatively more compact under the proposed ELM regulariser, which we now detail.

\section{ELM: Embedding and Logit Margins}
We now present ELM, a technique that augments margins in both logit
and embedding space.

\subsection{Formulation}

Consider a scorer $f_y( x ) = w_y^\top \Phi( x ) + b_y$,
where $w_y \in \mathbb{R}^K$ are the classification weights for label $y$,
$b_y \in \mathbb{R}$ is a bias term,
and $\Phi( x ) \in \mathbb{R}^K$ the learned embeddings for instance $x$.
The ELM objective is:
\begin{equation}
    \label{eqn:elm}
    \min_{w, \Phi}
    \frac{1}{N} \sum_{(x,y) \in S}
    \left[ \ell_{\rm mar}( y, f( x ) ) + \lambda \cdot  \Omega_{\rm pull}( x, y )   \right],
\end{equation}
where
$\ell_{\rm mar}$ is per~\eqref{eqn:xent-margin},
and $S\defEq\{(x_{i},y_{i})\}_{i=1}^{N}\stackrel{\text{i.i.d.}}{\sim}\bbP^{N}$ is a training sample from the joint distribution $\bbP$.
Furthermore, $\lambda \geq 0$ controls the trade-off between logit margin
(promoted by $\ell_{\rm mar}$)
and embedding margin (encouraged by $\Omega_{\rm pull}$).
Inspired by objectives in metric learning~\citep{Weinberger:2009},
we shall consider
\begin{align*}
    \Omega_{\rm pull}( x, y ) &\defEq \log\left[ 1 +
    \sum\nolimits_{\xPos \in S_y \backslash \{x\}} e^{\| \Phi( x ) - \Phi( \xPos ) \|_2^2 - \alpha_y} \right],
\end{align*}
where $S_y$ denotes the training samples with label $y$.
Intuitively, $\Omega_{\rm pull}$ acts to ``pull'' together embeddings from the same class.
Further, parameters $\{\alpha_y\}_{y \in [L]}$ serve as margins that control the desired slack in enforcing this consideration.
Intuitively, we seek to ensure that rare classes are pulled tightly together (i.e., small $\alpha_y$).
The pull regulariser $\Omega_{\rm pull}$ can be seen as a differentiable relaxation of
\begin{align*}
    \tilde{\Omega}_{\rm pull}( x, y ) &\defEq \max_{\xPos \in S_y \backslash \{x\}}
    \left[  \| \Phi( x ) - \Phi( \xPos ) \|_2^2 -\alpha_y \right]_+,
\end{align*}
which pulls together embeddings
of the same class $y$ so that, on average, each pair is no more than $\alpha_y$ away.
Observe that $\tilde{\Omega}_{\rm pull}(x,y) < \Omega_{\rm pull}(x,y)$.

With more uncertainty associated with rare
classes, it is reasonable to pull their embeddings together more strongly than those from frequent classes.
This implicitly ensures  that embeddings
of rare classes are well-separated from other classes and helps accommodate embeddings of unobserved instances during test time, which may have high variance.
In line with these, we thus propose setting $\alpha_y \propto \Pr(y)^{a}$ where $a > 0$.
In the sequel, we shall focus on $a=1$ or $a  = \frac{1}{2}$.
A similar consideration was made by Samuel and Chechik \cite{SamChe2021},
as shall be detailed in~\S\ref{sec:existing_work}.

\subsection{Connection to Existing Work}
\label{sec:existing_work}
The core elements of the ELM objective~\eqref{eqn:elm} are not without precedent.
For example, the idea of regularising embeddings has been widely explored in the area of contrastive learning~\citep{Wu:2018,Oord:2018,Khosla:2020}.
Similarly, the idea of having the pull term underpins Fisher linear discriminant analysis~\citep{Fisher:1936}.
However, the key to ELM's success in long-tailed problems is enforcing \emph{margins} in both the logit and embedding space, and having these margins be sensitive to the label distribution {$\mathbb{P}(y)$}.
We now
detail the relevant strands of prior work,
and
delineate the key differences to ELM.
(See Table~\ref{tbl:summary}.)

\begin{table*}[!t]
    \centering

    \renewcommand{\arraystretch}{1.25}

    \caption{%
    Summary of approaches to learning with logit and embedding margins.
    Here,
    ``long tail'' refers to whether or not the method explicitly accounts for skew in the label distribution;
    ``logit loss'' refers to whether or not the method explicitly learns logits for classification, as opposed to relying on a $k$-NN classifier;
    and ``embedding loss'' refers to whether or not the method explicitly regularises the embeddings in some way, as opposed to purely operating on logits.
    In long-tail settings,
    enforcing a logit margin is essential to ensure consistency for the balanced error,
    while enforcing an embedding margin is essential to ensure compactness of the embeddings on tail classes.
    Entries marked ``---'' are not applicable.}
    \label{tbl:summary}
    \vspace{2mm}
  \resizebox{\linewidth}{!}{%
    \begin{tabular}{@{}p{10cm}p{0.325in}p{0.5in}p{0.5in}p{0.8in}p{0.8in}@{}}
        \toprule
        \textbf{Method} & \textbf{Long tail?} & \textbf{Logit loss?} & \textbf{Logit margin?} & \textbf{Embedding loss?} & \textbf{Embedding margin?} \\
        \toprule
        Contrastive loss~\citep{Yi:2014,Wu:2018,Oord:2018,He:2019,Chen:2020} & --- & --- & --- & \tick & \cross \\
        Supervised contrastive loss~\citep{Khosla:2020,Chuang:2020} & \cross & \cross & --- & \tick & \cross \\
        Triplet loss~\citep{Weinberger:2009,Hadsell:2006,Schroff:2015,Sohn:2016} & \cross & \cross & --- & \tick & \tick \\
        Spreadout~\citep{Zhang:2017b} & \cross & \tick & \cross & \tick & \cross \\
        Center loss~\citep{Wen:2016} & \cross & \tick & \cross & \tick & \cross \\
        Hybrid contrastive learning~\citep{Liu:2020b} & \cross & \tick & \cross & \tick & \cross \\
        Softmax with margin~\citep{Cao:2019,Tan:2020,Ren:2020,Menon:2013,Wang:2021b} & \tick & \tick & \tick & \cross & --- \\
        Range loss~\citep{Zhang:2017} & \tick & \tick & \cross & \tick & \tick \\
        DRO-LT~\citep{SamChe2021} & \tick & \tick & \cross & \tick & \tick \\
        \midrule
        \rowcolor{cyan!10}
        Ours & \tick & \tick & \tick & \tick & \tick \\
        \bottomrule
    \end{tabular}%
    }
\end{table*}

\paragraph{Contrastive learning}
Contrastive learning~\citep{Wu:2018,Yi:2014,Oord:2018,He:2019,Chen:2020}
techniques seek to learn good representations $\Phi \colon \calX \to \mathbb{R}^K$
by aligning similar instances (e.g., a sample and its perturbation),
and pushing apart dissimilar instances (e.g., pairs of random samples).
This may be achieved by minimising $\Omega_{\rm con}( x ) \defEq $
$$  \mathbb{E}_{\xPos, \mathscr{N}( x )} \log\left[ 1 + \sum\nolimits_{\xNeg \in \mathscr{N}( x )} e^{\Phi( x )^\top \Phi( \xNeg ) - \Phi( x )^\top \Phi( \xPos )} \right], $$
where $\xPos$ is a ``positive'' sample for $x$,
and $\mathscr{N}( x )$ comprises contrasting ``negative'' samples for $x$.
In standard contrastive learning, there is no explicit supervision, and so $\mathscr{N}( x )$ may be taken to be randomly sampled inputs.
In supervised contrastive learning~\citep{Khosla:2020,Chuang:2020},
it is assumed that label information is present, and
$\mathscr{N}( x )$ comprises samples with a different label than $x, \xPos$.
Such techniques do \emph{not} involve a logit margin, and are not adapted to long-tail settings.

Objectives based on the triplet loss take a similar form~\citep{Weinberger:2009,Hadsell:2006,Schroff:2015,Sohn:2016}, with the contrasting set $\mathscr{N}( x )$ comprising one or more negative samples, typically chosen based on some form of negative mining.
Most such objectives enforce an explicit margin, i.e., for $\gamma > 0$,
$$ \Omega_{\rm trip}( x ) \defEq \mathbb{E}_{\xPos, \xNeg} [ \gamma + \Phi( x )^\top \Phi( \xNeg ) - \Phi( x )^\top \Phi( \xPos ) ]_+. $$
Here, $\gamma$ is constant across all samples,
and is thus not attuned to skewed label distributions.

\paragraph{Classification-contrastive hybrids}
Recently,~\citet{SamChe2021} proposed DRO-LT,
which adds the regulariser $\Omega_{\rm dro}( x, y ) \defEq$
\begin{equation}
    \label{eqn:dro-lt}
    \log\left[ \sum_{(x', y') \in S} e^{-\| \Phi( x' ) - \mu_y \|_2^2 + \|  \Phi( x ) - \mu_y \|_2^2 + \epsilon_y \cdot \bbI( y' \neq y )} \right],
\end{equation}
where $S$ is the set of all instance-label pairs, $\epsilon_y \propto 1 /
{\sqrt{\Pr( y )}}$,
and $\mu_y$ is the centroid in the embedding space of all samples in class $y$.
\citet{Wang:2021} proposed a similar loss with $\epsilon_y = 0$.
Like our ELM method, DRO-LT explicitly seeks to improve the quality of learned embeddings for tail classes.
However, there are important distinctions:

\begin{enumerate}[label=(\roman*),itemsep=0pt,topsep=5pt,leftmargin=16pt]
    \item DRO-LT is somewhat pessimistic, in that it pushes away the embedding for a sample $( x, y )$ to \emph{all} other samples $( \xPos, y' )$, regardless of whether $y' = y$.
    (Samples with $y' \neq y$ are however subject to a margin of $\epsilon_y > 0$.)
    As demonstrated in Figure~\ref{fig:intra},
    this can cause the embeddings for a given class to be \emph{more} spread out compared to ERM.
    By contrast,
    we only separate samples from different classes (by the logit-adjusted cross-entropy term),
    and pull together samples from the same class.

    \item we give a unified treatment of margins in both logit and embedding space.
    In particular, we justify our approach in terms of approximation to the Bayes solution (\S\ref{sec:analysis}).
\end{enumerate}

\citet{Liu:2020b} proposed to combine the softmax cross-entropy with a contrastive-like term:
$$ \Omega( x, y ) \defEq \log\left[ 1+\sum\nolimits_{\xNeg \in \mathscr{N}( x )}e^{f_{y}(\xNeg)-f_{y}(x)} \right]. $$
A similar objective was also considered in~\citet{Veit:2020}.
Compared to our approach,
there are two key distinctions.
First,
there is no margin enforced in either term.
Second,
the contrastive term operates in logit space,
and thus changes the target function in a non-trivial manner;
in a long-tail setting, this would erase the consistency guarantees for the balanced error~\citep{Menon:2021}.

\paragraph{Improved embeddings for tail classes}
For long-tail settings,
some works have considered means of
improving embeddings for tail classes.
For example, in~\citet{Zhang:2017}, it was proposed to minimise
$$ \Omega_{\rm range}( x ) \defEq \max_{y \neq y'} [ \gamma - \| \mu_{y} - \mu_{y'} \|_2^2 ]_+, $$
so that different classes' centroids are pushed apart.
There are three important points worth mentioning.
First, it is based on a hard max, which allows for limited gradient propagation.
Second, the margin $\gamma$ is the same for all labels, and is not attuned to tail classes.
Third, this does not consider logit margins, which we demonstrate can lead to suboptimal decision boundaries.

\citet{Yin:2018,Liu:2019} proposed to transfer information from dominant to rare class embeddings directly.
This is an interesting yet orthogonal consideration to improving the spread and separation of tail embeddings;
its fusion with the ideas of the present paper would be of interest in future work.
For discussion of additional related work, see Appendix~\ref{app:related}.

\section{Analysis: Why Does ELM help?}
\label{sec:analysis}
At its core, ELM enforces logit \emph{and} embedding margins.
Both of these help improve performance, as we now argue.

\subsection{Why Do Margins Help?}

The case for logit margins has already been made in prior work~\citep{Cao:2019,Tan:2020,Ren:2020,Menon:2021,Wang:2021b}, but is worth succinctly recapitulating.
There are two key arguments:
first, for generic supervised learning problems,
margin bounds~\citep{Bartlett:1998,Koltchinskii:2002,Bartlett:2017} establish that large margins imply good generalisation.
Second,
for long-tail problems in particular,
excluding logit margins would implicitly seek to model $\Pr( y \mid x )$;
absent further correction,
this solution will be suboptimal for the balanced error~\citep{Menon:2013,Collell:2016,Ren:2020,Menon:2021}.

Embedding margins are useful if we want to use a $k$-nearest neighbour classifier as a post-hoc training procedure.
Such post-hoc training procedures have proven successful in long-tail settings~\citep{Kang:2020}.
More fundamentally, however,
one may justify the regularisation of these embeddings from the perspective of approximating the Bayes-optimal decision boundary, as we now see.

\subsection{ELM and the Bayes-Optimal Classifier}

We now quantify the value of the pull term in ELM~\eqref{eqn:elm}.
\begin{prop}
\label{prop:pull_and_variance}
Let $\bar{\Omega}_{\mathrm{pull}}(y)\stackrel{.}{=}\frac{1}{|S_{y}|}\sum_{x\in S_{y}}\Omega_{\mathrm{pull}}(x,y)$
and
$\Phi(x)\stackrel{.}{=}\left(\Phi_{1}(x),\ldots,\Phi_{K}(x)\right)^{\top}$.
Then,
$$ \frac{2|S_{y}|}{|S_{y}|-1} \cdot \sum_{j=1}^{K}\hat{\mathbb{V}}[\Phi_{j}(x)\mid y]-\alpha_{y}+ \log\left(|S_{y}|-1\right)\le\bar{\Omega}_{\mathrm{pull}}(y). $$
\vspace{-4mm}
\end{prop}
Here,
$\hat{\mathbb{V}}[\cdot\mid y]$ is the empirical conditional variance of class $y$.
Proposition \ref{prop:pull_and_variance} (proof in Appendix~\ref{sec:proof_pull_and_variance}) states that
minimizing the class-wise average of the pulling objective
$\Omega_{\rm pull}$ will
also minimize the sum of class-conditional variances (in the embedding space) of all dimensions;
i.e., the pulling objective encourages a small intra-class
variance.
As shall be seen later in Proposition~\ref{prop:gen_bound_pull},
reduction of class-conditional variances directly translates to
better generalisation.
Expanding the result stated in Proposition \ref{prop:pull_and_variance}, we have
\begin{align*}
& \frac{1}{|S_{y}|}\sum\nolimits_{x\in S_{y}}\Omega_{\mathrm{pull}}(x,y) \\
& \ge  \frac{2|S_{y}|}{|S_{y}|-1} \sum\nolimits_{x\in S_{y}}\|\Phi(x)-\mu_{y}\|^{2}-\alpha_{y}+ \log\left(|S_{y}|-1\right),
\end{align*}
where $\mu_y$ is the centroid of all embeddings in class $y$.
The first term on the right hand side is exactly the regulariser in the center loss~\citep{Wen:2016}.
This may be interpreted as encouraging a more Gaussian distribution for the embeddings:
indeed,
\begin{align*}
& \sum\nolimits_{x \in S_y} \| \Phi( x ) - \mu_y \|_2^2 \\
& = \sum\nolimits_{x \in S_y} -\log
\mathsf{N}( \Phi( x ); \mu_y, \sigma^2_y I ) + \textrm{constant},
\end{align*}
i.e.,
it is the log-likelihood under an isotropic Gaussian model for $\mathsf{Z} \mid \mathsf{Y} = y$, where we define the random vector $\mathsf{Z} \defEq \Phi( \mathsf{X} ) $ and $\mathsf{X} \sim \Pr( x )$.
The value of such a model is that
\emph{it justifies the use of a softmax distribution for the final layer}.
In particular,
when $\mathsf{Z} \mid \mathsf{Y} = y \sim \mathsf{N}( \mu_y, \sigma^2_y I )$,
\begin{align*}
    \Pr( y \mid z ) &\propto \Pr( z \mid y ) \cdot \Pr( y ) \\
    & \propto \exp\left( \Phi( x )^\top \frac{\mu_y}{\sigma^2_y} - \frac{\| \mu_y \|_2^2}{2 \sigma^2_y} + \log \Pr(y) \right),
\end{align*}
where $z = \Phi(x)$ denotes a realization of $\mathsf{Z}$.
Thus, under a Gaussian distribution for embeddings, we may perfectly express $\Pr( y \mid x )$ as an affine function of $\Phi( x )$, composed with a softmax link function.

\subsection{Generalisation Bound}

We now present a generalisation bound of the logit-adjusted cross-entropy loss \eqref{eqn:xent-margin} in Proposition~\ref{prop:gen_bound_pull}.
We consider the binary case where $y \in \{-1,+1\}$.
In this case, it is sufficient to consider a real-valued scorer
$f(x) \stackrel{.}{=} w^\top \Phi(x) + b \in \mathbb{R}$ that computes the  logit
for class $+1$.
Accordingly, the logit-adjusted cross-entropy loss in \eqref{eqn:xent-margin} can
be written as
$\ell_{\mathrm{log}}(y,f(x)+\Delta_{y}):=\log\left[1+e^{-y(f(x)+\Delta_{y})}\right]$.

\begin{prop}
\label{prop:gen_bound_pull}
Let $\Delta_{y}\in\mathbb{R}$, and
$\ell_{\mathrm{log}}(y,f(x)+\Delta_{y}):=\log\left[1+e^{-y(f(x)+\Delta_{y})}\right]$.
Suppose $y\in\{-1,1\}$, and $\sup_{x\in\mathcal{X},y\in\{-1,1\}}\ell_{\mathrm{log}}(y,f(x)+\Delta_{y})\le B$
for some $B\in(0,\infty)$. Then, given $f(x)\stackrel{.}{=}w^{\top}\Phi(x)+b\in\mathbb{R}$,
with probability at least $1-\delta$,
\begin{align*}
 & \mathbb{E}_{(x,y)\sim\mathbb{P}_{xy}}\ell_{\mathrm{log}}(y,f(x)+\Delta_{y})\\
 & \le\frac{7B\ln2/\delta}{3}\sum_{y\in\{-1,1\}}\frac{\mathbb{P}(y)}{|S_{y}|-1}
  +\sum_{y\in\{-1,1\}}\frac{\mathbb{P}(y)}{|S_{y}|}\sum_{x\in S_{y}}\ell_{\mathrm{log}}(y,f(x)+\Delta_{y})\\
 & \phantom{\le}+\|w\|\sqrt{\ln\frac{2}{\delta}\sum_{y\in\{-1,1\}}\frac{\mathbb{P}(y)}{|S_{y}|}\left[\bar{\Omega}_{\mathrm{pull}}(y)+\alpha_{y}\right]},
\end{align*}

where $\bar{\Omega}_{\mathrm{pull}}(y)\stackrel{.}{=}\frac{1}{|S_{y}|}\sum_{x\in S_{y}}\Omega_{\mathrm{pull}}(x,y)$.
\end{prop}

Proposition \ref{prop:gen_bound_pull} suggests that the generalisation gap of
classifiers trained with the softmax cross-entropy with logit margins (see
\eqref{eqn:xent-margin}) can be expressed as a function of the proposed pull
term $\Omega_{\mathrm{pull}}$. This justifies its use in the ELM (see
\eqref{eqn:elm}) as encouraging better generalisation.
The improvement from adding a pull term to the logit-adjusted loss
results in a larger gap between the per-class logits, and a more
compact per-class distribution of embeddings, as shall be seen in \S\ref{sec:experiments}.

Note that the difference between
$\frac{1}{N}\sum_{(x,y) \in S} \ell_{\mathrm{log}}(y,f(x)+\Delta_{y})$
and
$\sum_{y\in\{-1,1\}}\frac{\mathbb{P}(y)}{|S_{y}|}\sum_{x\in S_{y}}\ell_{\mathrm{log}}(y,f(x)+\Delta_{y})$
vanishes as $N \to \infty$.
The former thus approximates the empirical error
of the logit-adjusted loss $\ell_{\rm mar}$, and aligns with our formulation in \eqref{eqn:elm}.
Further note that while the generalisation error on the left-hand side
of the bound is with respect to the joint distribution (which depends on the skewed label distribution $\bbP(y)$), it is consistent for minimising the balanced error where the label distribution is uniform when $\Delta_y = \log \frac{\bbP(y=+1)}{\bbP(y=-1)}$ \citep{Menon:2021}.

\section{Experiments on Long-tail Benchmarks}
\label{sec:experiments}

\begin{table*}[t]
    \centering
    \renewcommand{\arraystretch}{1.25}

    \caption{
    Test set accuracy (averaged over $3$ trials) on real-world datasets.
    Here, $^\star$, $^\ddagger$, $^\dagger$, $^\diamond$ are numbers for
    ``$\tau$-normalised'' from~\citet[Table 3, 7]{Kang:2020};
    ``Class-Balanced'' from~\citet[Table 2, 3]{Cui:2019};
    ``LDAM + DRW'' from~\citet[Table 2, 3]{Cao:2019};
    ``DRO-LT'' from~\citet[Table 1, 2]{SamChe2021}.
    The \fade{faded} cells denote ``multi-stage'' methods requiring training multiple models.
    The others, including our proposed ELM, perform training in one-stage. Our goal is to understand the performance we can get from a one-stage training procedure.
    }
    \label{tbl:results}
    \vspace{2mm}

    \resizebox{0.975\linewidth}{!}{
    \begin{tabular}{@{}lllll@{}}
        \toprule
        \textbf{Method} & \textbf{CIFAR10-LT} & \textbf{CIFAR100-LT} & \textbf{ImageNet-LT} & \textbf{iNaturalist} \\
        \toprule
        Cross-entropy (CE)                 &
        72.84 &
        38.36 &
        45.20 &
        61.34 \\
        CB Focal~\citep{Cui:2019}             &
        74.57$^\ddagger$ &
        39.60$^\ddagger$ &
        46.79 & %
        64.16$^\ddagger$ \\
        LDAM + DRW~\citep{Cao:2019}
        &
        77.03$^\dagger$ &
        42.04$^\dagger$ &
        50.15 &
        68.00$^\dagger$ \\
        LogAdj~\citep{Ren:2020,Menon:2021,Wang:2021b} &
        {77.67}      &
        {43.89} &
        50.37 &
        {66.36} \\
        CE + DRO-LT (one-stage)~\citep{SamChe2021} &
        72.70 &
        41.98 &
        45.70 &
        62.05 \\
        \midrule
        \fade{CE + weight normalisation (multi-stage)}~\citep{Kang:2020} &
        \fade{78.50} &
        \fade{41.34} &
        \fade{50.63} &
        \fade{65.60$^\star$} \\ %
        \fade{CE + DRO-LT (multi-stage)}~\citep{SamChe2021} &
        \fade{80.50} &
        \fade{46.92$^\diamond$} &
        \fade{53.00$^\diamond$} &
        \fade{69.70$^\diamond$} \\
        \midrule
        \textbf{ELM} (proposed, one-stage) &
        77.95 &
        45.77 &
        50.60 &
        68.71 \\
        \bottomrule
    \end{tabular}
    }
\end{table*}

We present results confirming that
ELM performs well on benchmarks for long-tail learning.

\paragraph{Datasets}
We present results on image classification benchmarks for long-tail learning:
CIFAR10-LT, CIFAR100-LT, ImageNet-LT and iNaturalist 2018.
Each of these datasets has a skewed training set, and balanced test set.
The long-tailed (``LT'') CIFAR datasets
are constructed by downsampling labels
from the original CIFAR train sets,
following the
{\sc Exp} profile of~\citet{Cui:2019,Cao:2019} with imbalance ratio
$\rho = {\max_{y} \Pr( y )} / {\min_{y} \Pr( y )} = 100$.
The long-tailed ImageNet dataset is as constructed in~\citet{Liu:2019},
and iNaturalist as per~\citet{VanHorn:2017}.

\paragraph{Models}
We employ a CIFAR-ResNet-32 for the CIFAR datasets,
and a ResNet-50 for ImageNet and iNaturalist.
See Appendix~\ref{app:hyperparameters} for details on training hyper-parameters, which follow~\citet{Menon:2021}.

\paragraph{Baselines}
We compare the proposed ELM method~\eqref{eqn:elm}
in terms of \emph{balanced} test set accuracy
against several baselines:
\begin{enumerate*}[label=(\roman*)]
    \item cross-entropy ({\bf CE}) minimisation;
    \item the class-balanced ({\bf CB Focal}) loss of~\citet{Cui:2019}, which applies asymmetric weights on the per-class losses;
    \item {\bf LDAM+DRW}~\citep{Cao:2019}, which enforces a logit margin;
    \item the logit adjustment ({\bf LogAdj}) loss~\citep{Ren:2020,Menon:2021,Wang:2021b}, which enforces a logit margin per~\eqref{eqn:xent-margin};
    \item {\bf DRO-LT}~\citep{SamChe2021}, which enforces an embedding margin $\epsilon_y \propto {\Pr(y)}^{-1/2}$.
\end{enumerate*}
For ELM, we set the logit margin $\Delta_{yy'} = \log \frac{\Pr( y' )}{\Pr( y )}$,
following the {\bf LogAdj} loss;
thus, ELM imposes additional regularisers over this method.
We detail the choice of $\alpha_y$ in the Appendix.

\paragraph{One- versus multi-stage methods}
We remark here that
DRO-LT is a ``multi-stage'' method, as it requires first obtaining centroid estimates $\mu_y$ from CE training;
training using these, plus the DRO-LT regulariser~\eqref{eqn:dro-lt};
and finally training a balanced classifier on the resulting embeddings.
By contrast, all other baselines --- and the proposed ELM --- are ``one-stage'' methods.
Thus, for an equitable comparison, we consider a ``one-stage'' version of DRO-LT, which does not have a separate estimation of $\mu_y$, nor retraining of a linear model on the learned embeddings.
For completness, we additionally quote the results of
multi-stage DRO-LT from~\citet{SamChe2021}.
As another multi-stage baseline, we report the results of
post-hoc weight normalisation~\citep{Kang:2020} on the CE solution.

An important goal of our experiments is to understand the performance (balanced accuracy) we can get from training in one stage.
This helps understand losses for long-tailed learning without being occluded by benefits gained from the more generally applicable augmented procedures such as balanced sampling, and fine-tuning the classifier layer.

\begin{figure*}[t]
    \centering

    \resizebox{0.99\linewidth}{!}{
    \subcaptionbox{CIFAR100-LT (100 classes).}{\includegraphics[scale=0.2]{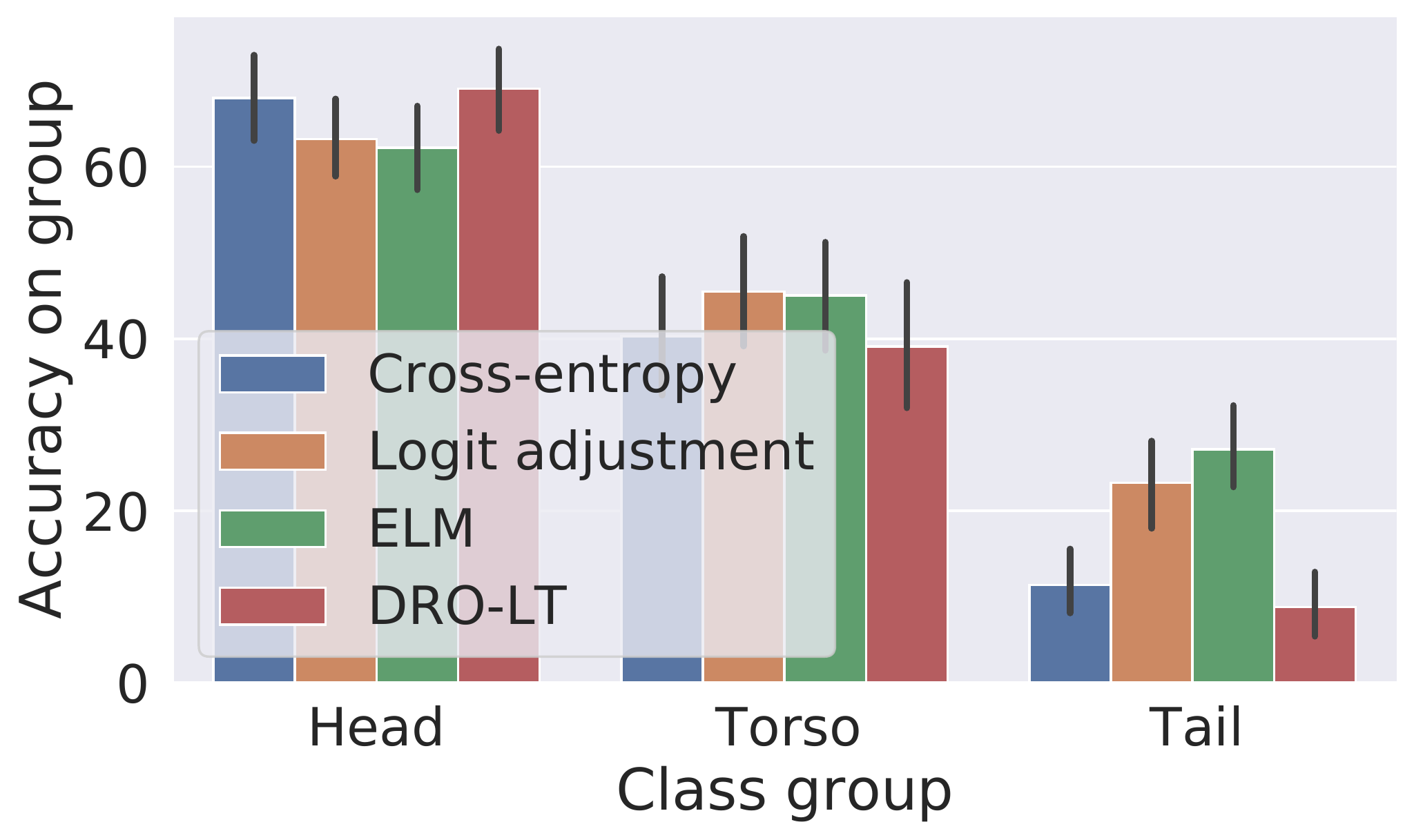}}
    \quad
    \subcaptionbox{Imagenet-LT (1000 classes).}{\includegraphics[scale=0.2]{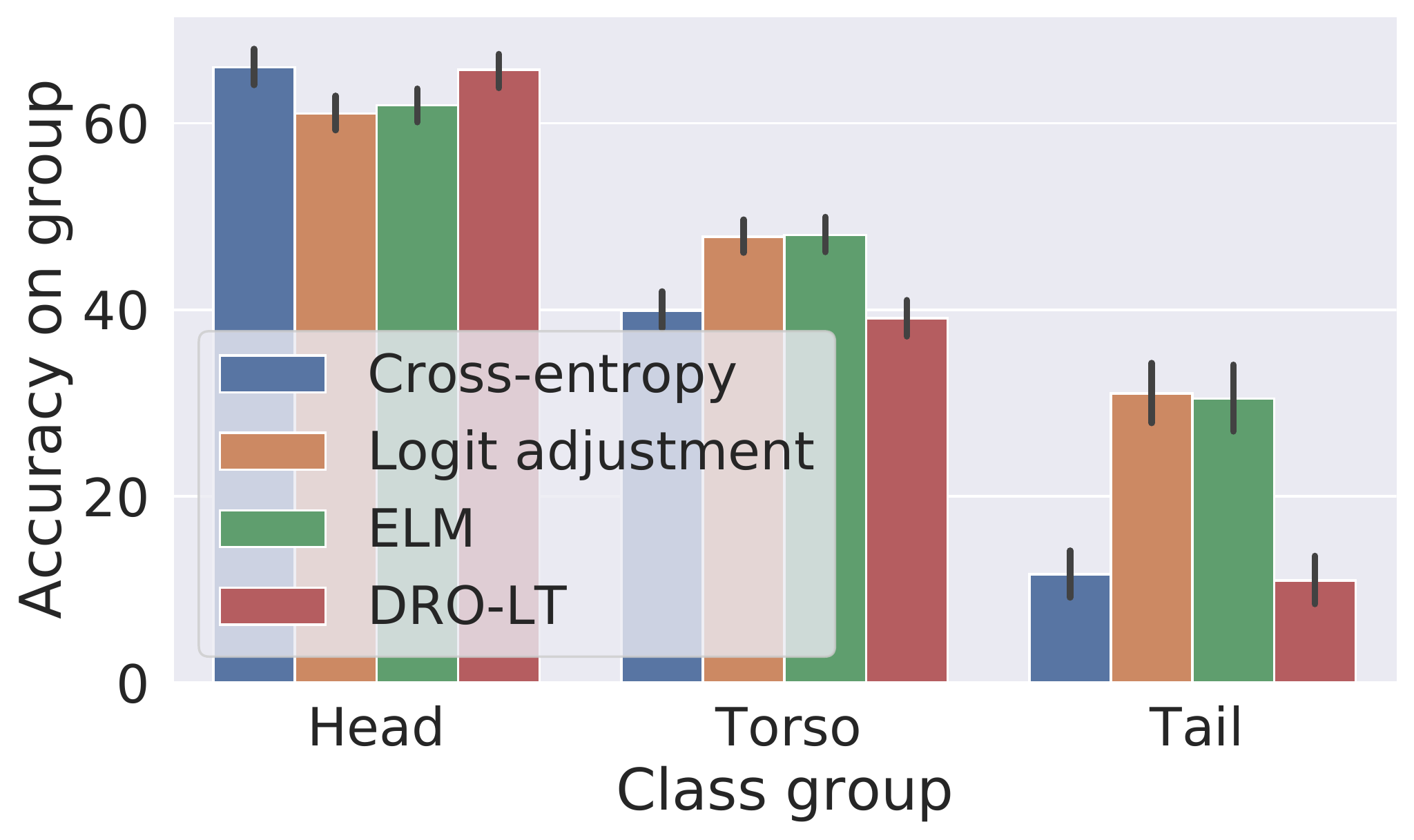}}
    \quad
    \subcaptionbox{iNaturalist (8142 classes).}{\includegraphics[scale=0.2]{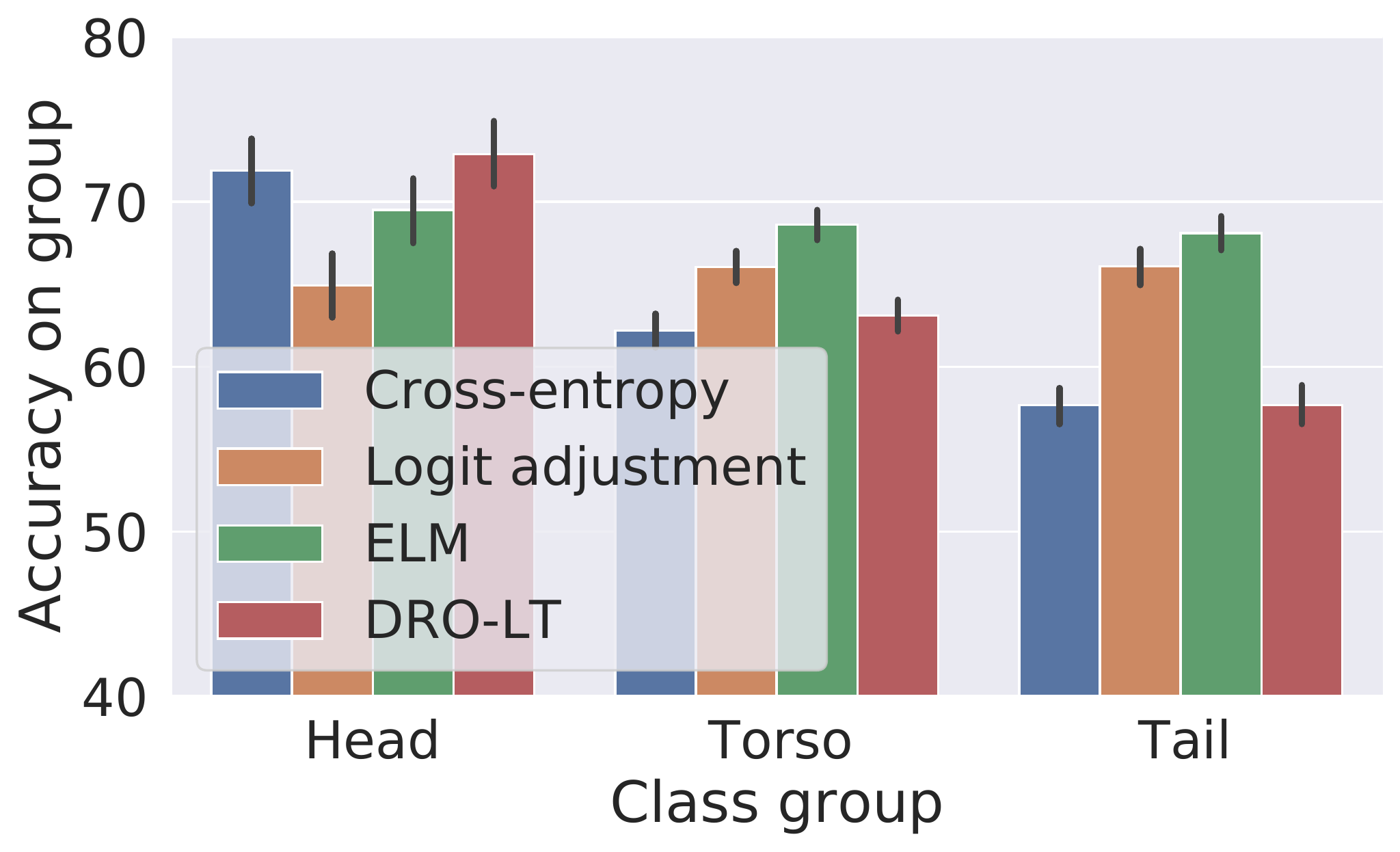}}
    }

    \caption{Breakdown of per-label accuracies, where labels are bucketed into groups based on their frequency.
    Compared to enforcing logit margins,
    additionally
    enforcing embedding margins are seen to help
    as ELM improves upon LogAdj on the Torso and Tail groups.
}
    \label{fig:breakdown}
\end{figure*}
\paragraph{Results and analysis}
Table~\ref{tbl:results} presents the results on all datasets.
We make some key observations.
\emph{Logit and embedding margins help}.
In keeping with prior work, approaches that enforce a logit margin, e.g., LDAM~\citep{Cao:2019}, and LogAdj~\citep{Ren:2020,Menon:2021,Wang:2021b},
perform significantly better than CE.
Similarly, one-stage CE + DRO-LT,
which enforces an embedding margin,
consistently improves over CE.

Compared to these techniques, combining both logit and embedding margins
yields improvements, as shown by the performance of ELM,
intuitively owing to it encouraging pulling together of similar embeddings.
Interestingly, even when compared to a multi-stage version of DRO-LT, the performance of our one-stage ELM remains favourable, with only a small difference across all datasets.

\emph{Breakdown of performance}.
The above illustrates the ELM can improve the overall tradeoff between rare and dominant labels.
For a more fine-grained understanding,
following~\citet{Kang:2020},
we break down the labels into three groups,
termed ``Head'' (labels with $\geq 100$ training samples),
``Torso'' ($[ 20, 100)$ samples), and
``Tail'' ($< 20$ samples).
Figure~\ref{fig:breakdown} reveals that, per~\citet{Menon:2021},
LogAdj achieves gains on both the Torso and Tail groups, at the mild expense of performance on the Head group.
Further adding an embedding margin via ELM yields gains on the Torso and Tail groups.
Interestingly, on the challenging iNaturalist data, there are gains on the Head group as well.

\begin{figure*}[t]
    \centering
    \resizebox{0.99\linewidth}{!}{
    \subcaptionbox{Label 10 (Head).}{
    \includegraphics[scale=0.25]{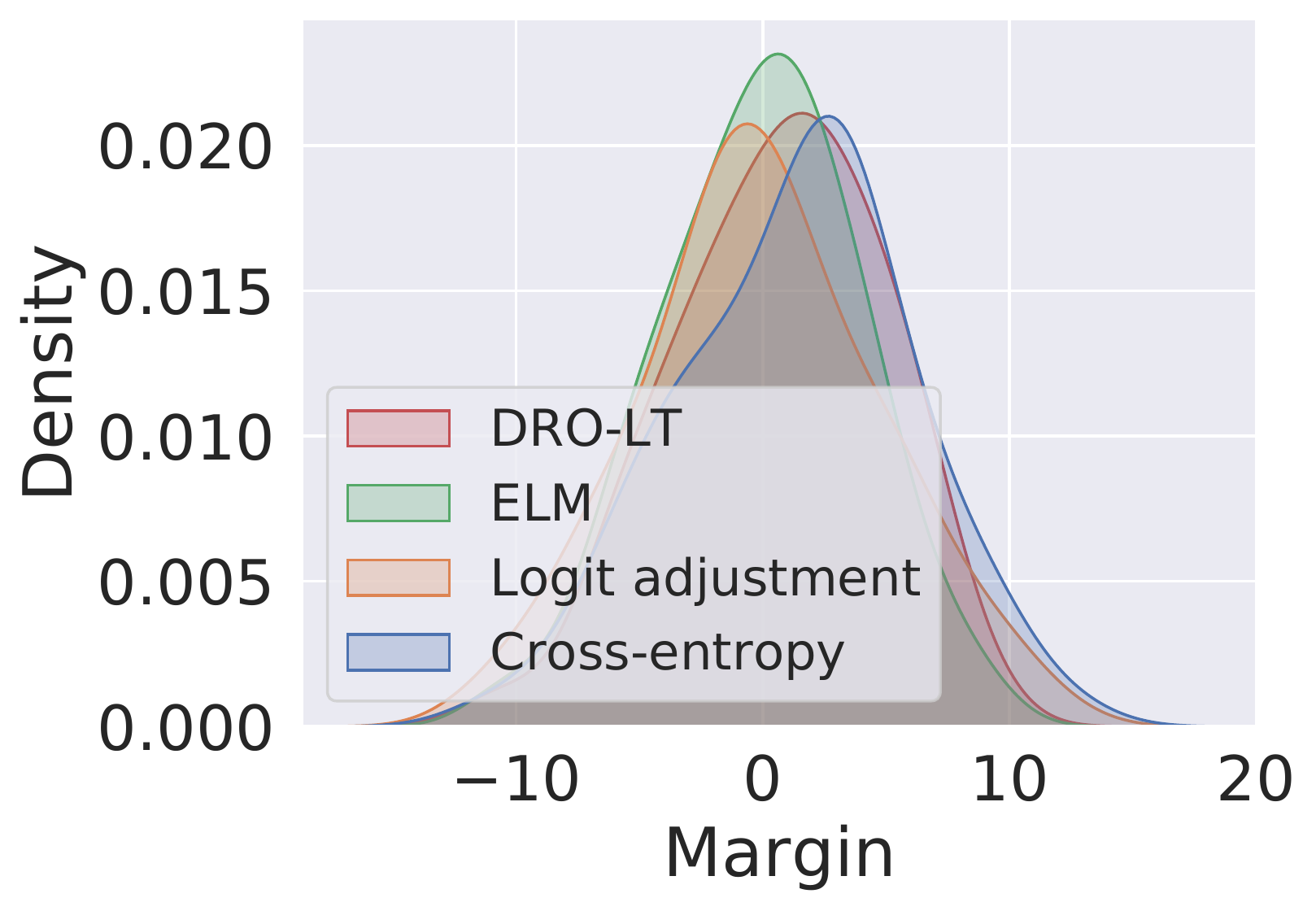}
    }
    \quad
    \subcaptionbox{Label 50 (Torso).}{
    \includegraphics[scale=0.26]{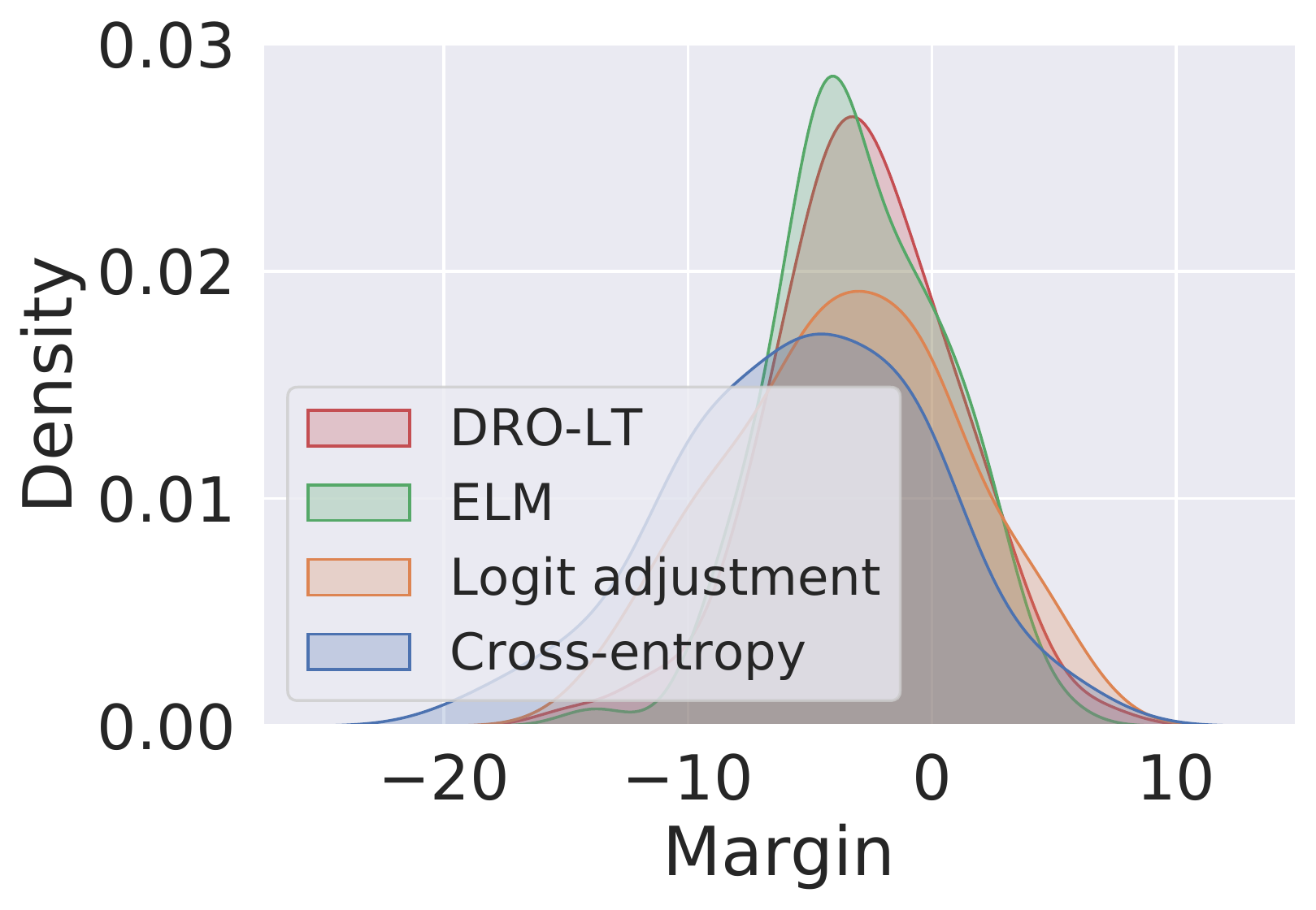}
    }
    \quad
    \subcaptionbox{Label 90 (Tail).}{
    \includegraphics[scale=0.25]{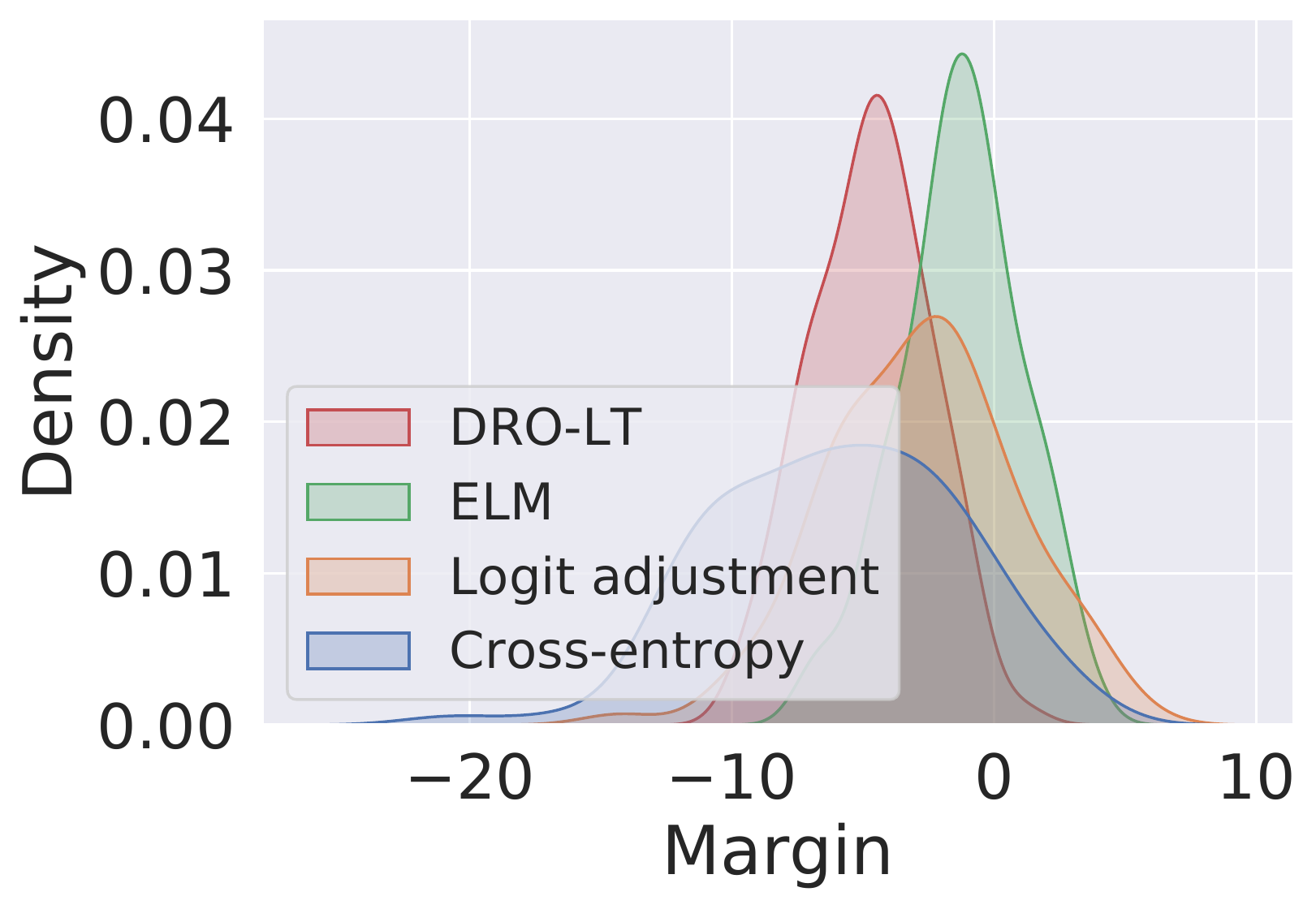}
    }
    }
    \caption{Margin plots on CIFAR100-LT.
    For a given label $y$, we plot the distribution of the margins $\gamma( x, y ) \defEq f_y( x ) - \max_{y' \neq y} f_{y'}( x )$ for instances $x$ with label $y$.
    ELM is seen to consistently reduce the variance of the margin distribution,
    while also shifting it favourably over CE on the Tail label.
    }
    \label{fig:margin}
    \vspace{-2mm}
\end{figure*}

\begin{figure*}[!t]
    \centering
    \resizebox{0.99\linewidth}{!}{
    \subcaptionbox{Label 10 (Head).}{
    \includegraphics[scale=0.24]{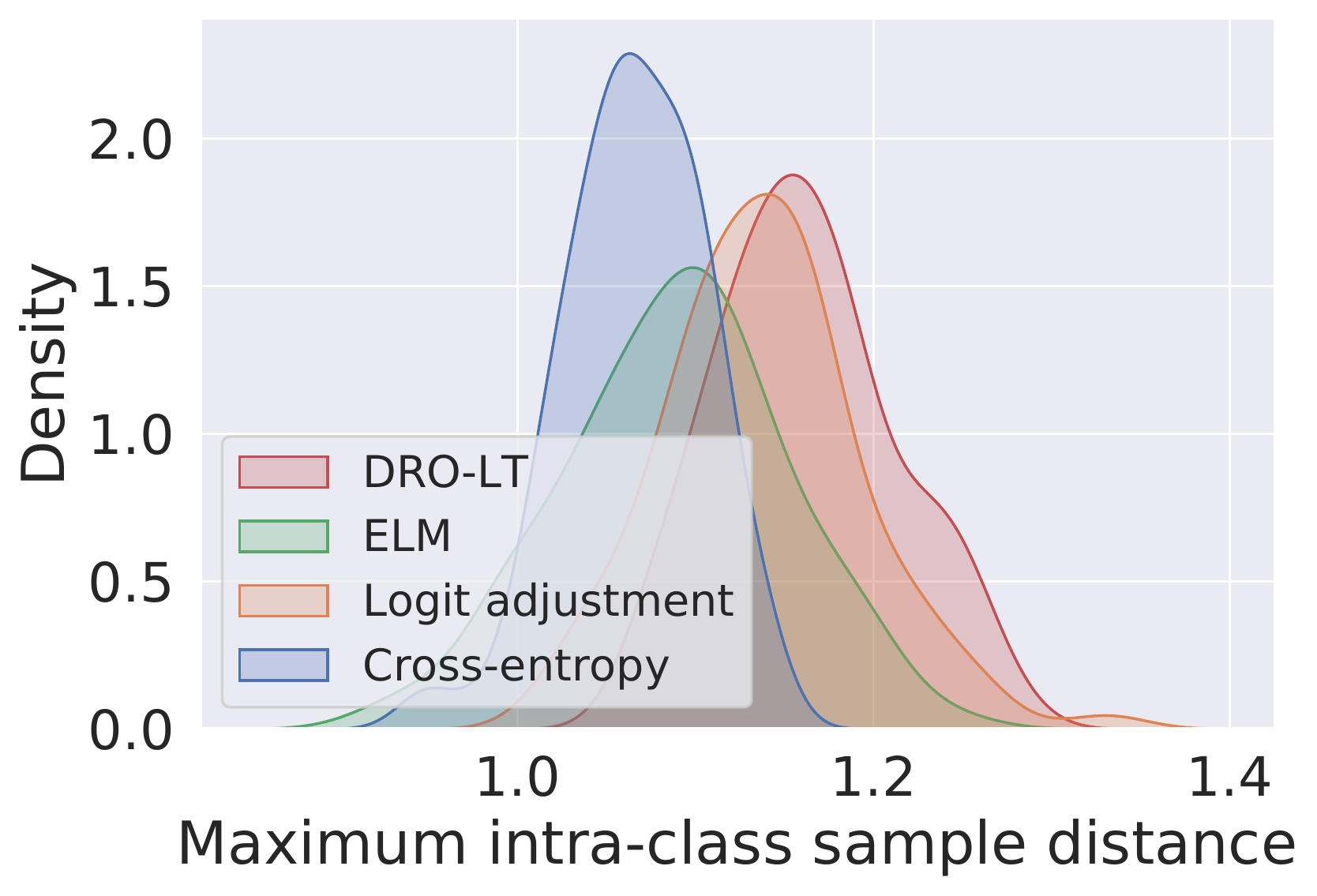}
    }
    \quad
    \subcaptionbox{Label 50 (Torso).}{
    \includegraphics[scale=0.24]{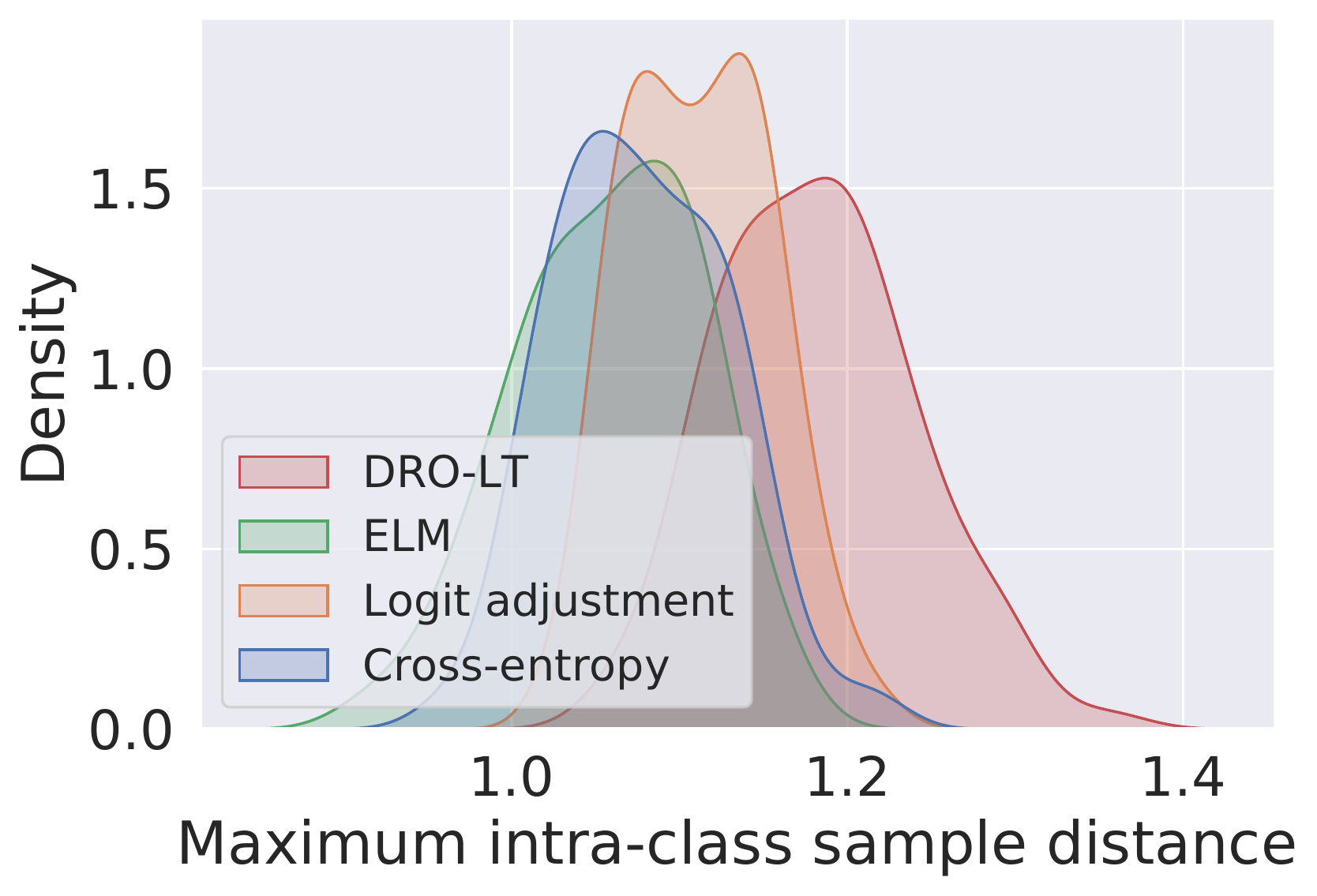}
    }
    \quad
    \subcaptionbox{Label 90 (Tail).}{
    \includegraphics[scale=0.24]{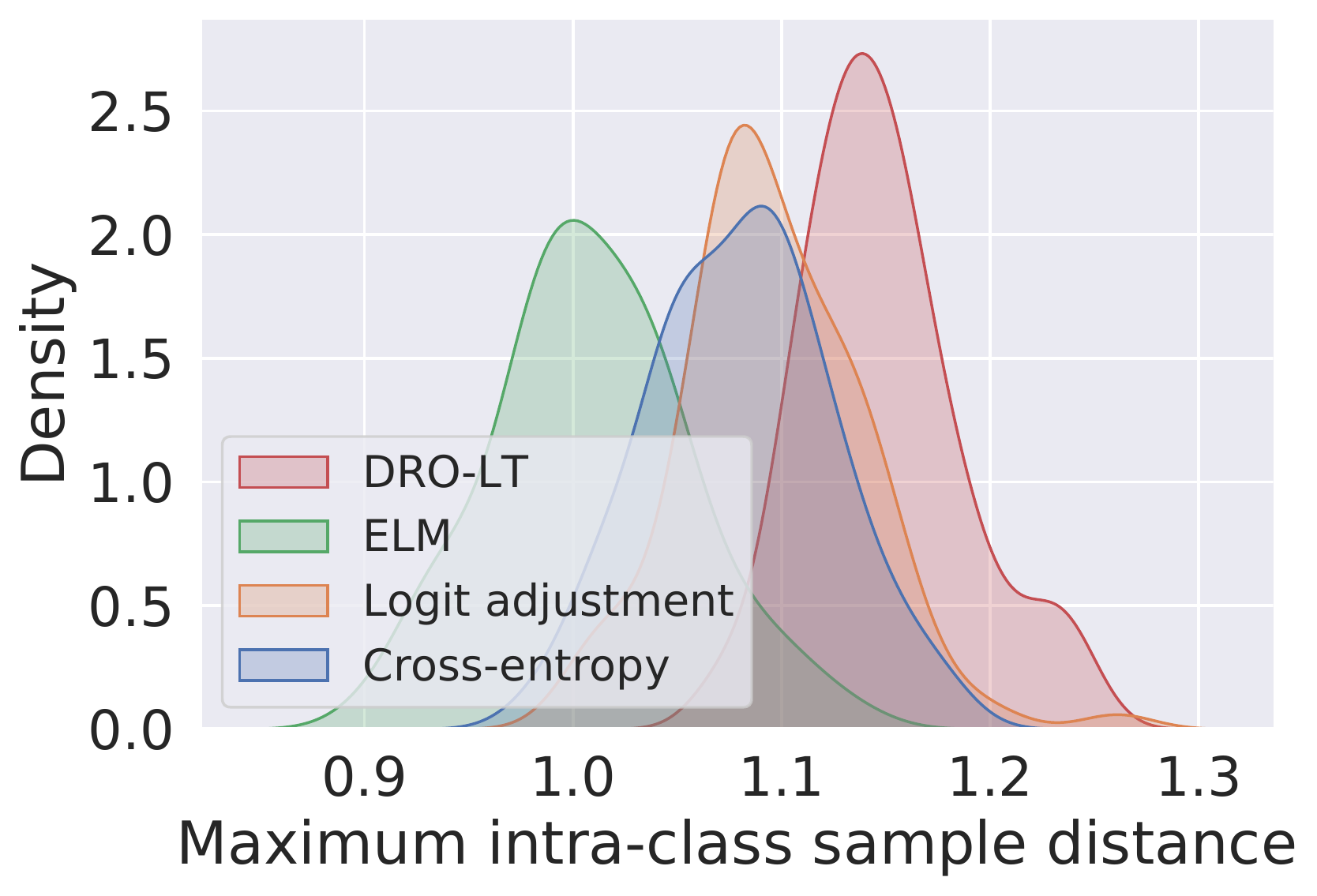}
    }
    }
    \caption{Maximum intra-class distances on CIFAR100-LT.
    For a given label $y$, we plot the distribution of the maximum intra-class distances $d_{\rm max}( x ) \defEq \max_{\xPos \in S_y} \| \Phi( x ) - \Phi( \xPos ) \|_2$ for instances $x$ with label $y$;
    distances are normalised by the maximal embedding norm $\max_{x \in S} \| \Phi( x ) \|_2$.
    DRO-LT is seen to increase these distances slightly, owing to contrasting samples within the same class.
    }
    \vspace{-3mm}
    \label{fig:intra}
\end{figure*}

\begin{figure*}[!t]
    \centering
    \subcaptionbox{LogAdj CE}{
    \includegraphics[width=0.25\textwidth, trim=0 0 140 0, clip]{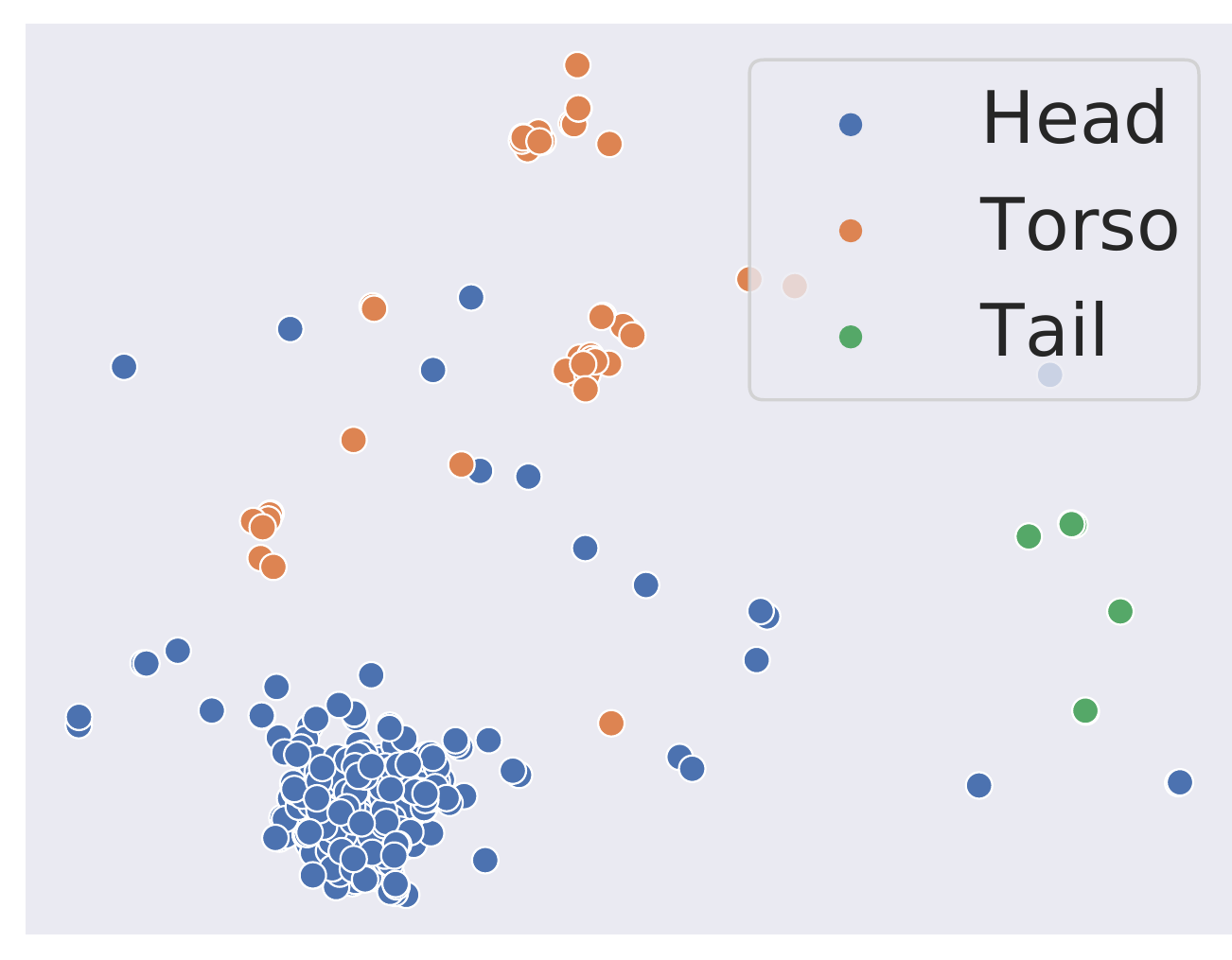}
    }
    \hspace{1cm}
    \subcaptionbox{ELM.}{
    \includegraphics[width=0.25\textwidth,trim=0 0 140 0, clip]{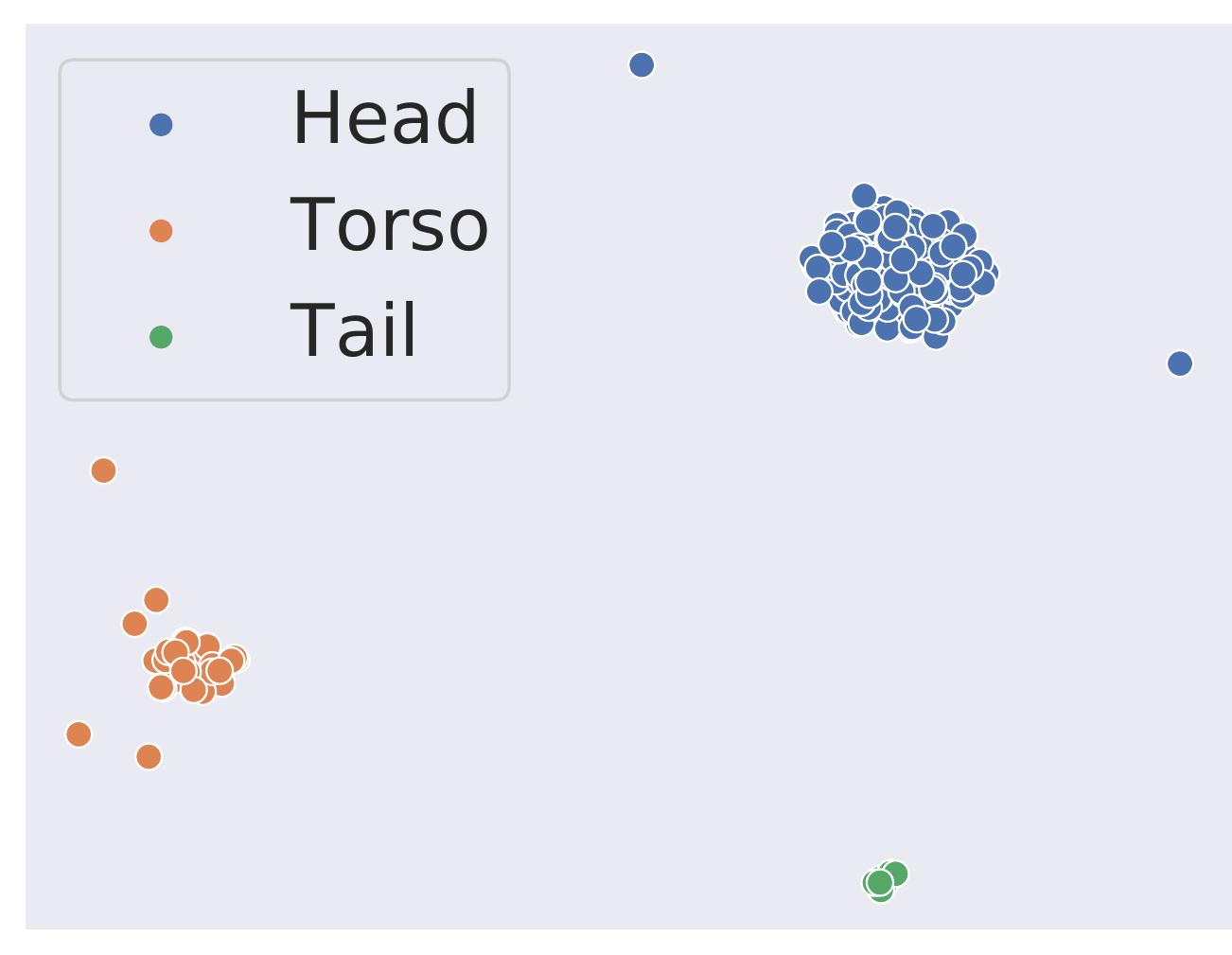}
    }
    \hspace{1cm}
    \subcaptionbox{DRO-LT.}{
    \includegraphics[width=0.25\textwidth,trim=0 0 140 0, clip]{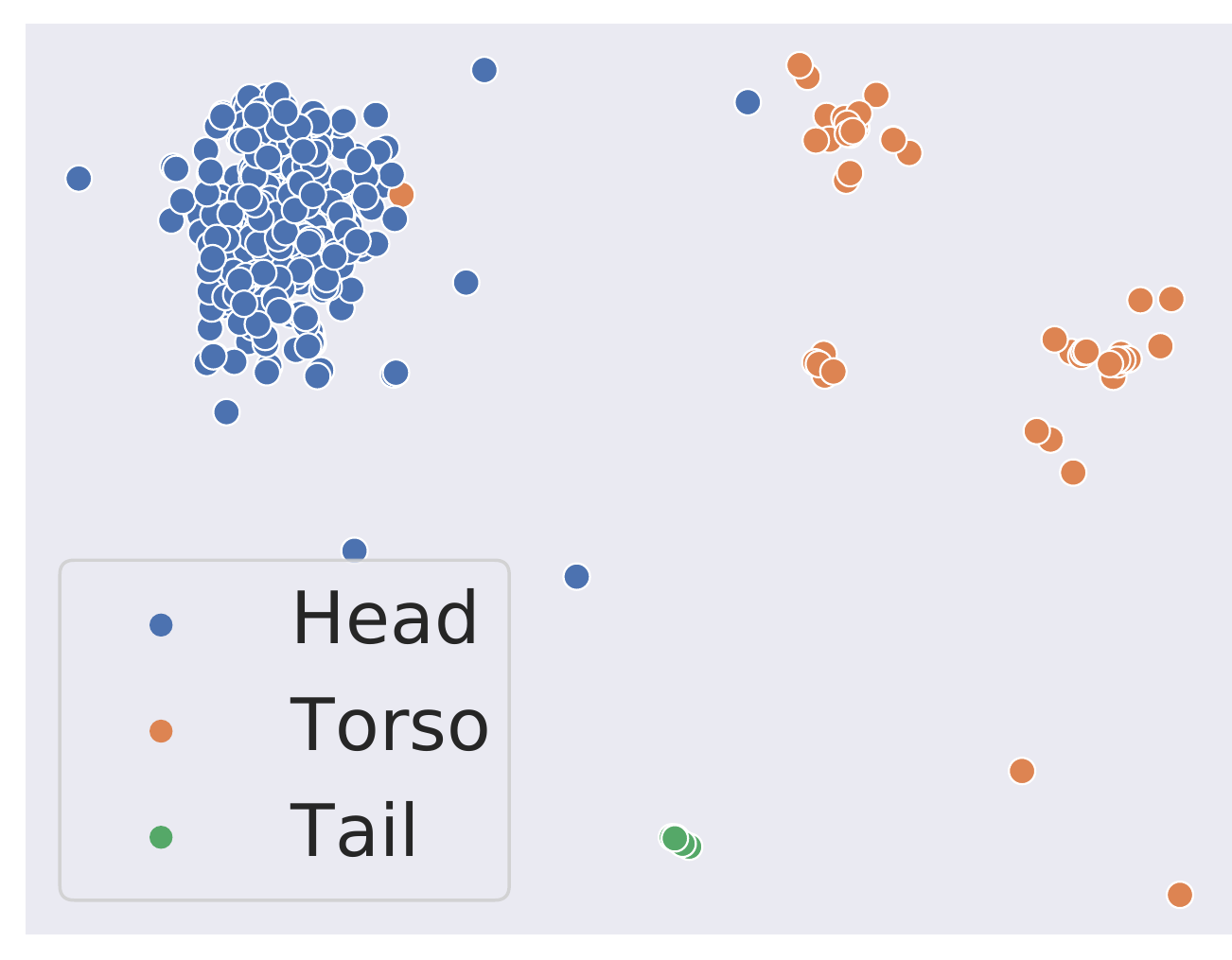}
    }
    \caption{Training set embedding visualisations on CIFAR100-LT via tSNE.
    Shown are a sample of classes that have $\geq 100$ associated samples (``Head''),
    $[20, 100)$ samples (``Torso''),
    and $< 20$ samples (``Tail'') in the training set.
    ELM produces more compact embeddings.
    }
    \label{fig:tsne}
    \vspace{-3mm}
\end{figure*}

\emph{Analysis of logit margins}.
We analyse the distribution of logit margins $\gamma( x, y ) \defEq f_y( x ) - \max_{y' \neq y} f_{y'}( x )$ on CIFAR100-LT.
We pick three labels from the Head, Torso, and Tail slices,
and compare the margin distributions for the cross-entropy loss (CE),
the logit adjustment (LogAdj) loss, the proposed ELM, and DRO-LT.
Figure~\ref{fig:margin} visualises these margin distributions.
As expected, logit adjustment tends to trade off performance on dominant classes, while significantly increasing margins on rare classes.

Interestingly,
from Figure~\ref{fig:margin},
ELM strongly controls the margin distribution, which becomes less variable.
Further, on the tail label, we see that ELM significantly shifts the mode of the margins over cross-entropy and LogAdj.
As such logit margins directly control generalisation performance~\citep{Bartlett:1998,Bartlett:2017},
this lends further credence to ELM improving classification performance as also suggested by Proposition~\ref{prop:gen_bound_pull}.

\emph{Analysis of embeddings}.
We conduct a similar analysis on the distances between the learned embeddings.
For instances $x$ with label $y$,
Figure~\ref{fig:intra} visualises the maximum intra-class distances $d_{\rm max}( x ) \defEq \max_{\xPos \in S_y} \| \Phi( x ) - \Phi( \xPos ) \|_2$.
The distances are normalised by the maximal embedding norm $\max_{x \in S} \| \Phi( x ) \|_2$.
We observe that DRO-LT increases the \emph{intra}-class distances slightly compared to CE, owing to contrasting samples within the same class.
By contrast, the pull part of ELM ensures that these distances remain small, thus encouraging tighter clusters, especially for tail classes.

To visually inspect the learned embeddings,
we create 2D tSNE~\citep{Maaten:2008} visualisations of the embeddings learned by LodAdj, ELM, and DRO-LT.
Figure~\ref{fig:tsne} illustrates these embeddings for a sample of
classes from the previously created Head, Torso and Tail buckets.
ELM is seen to produce more compact and separated embeddings compared to logit-adjusted cross-entropy minimisation and DRO-LT.
For additional experiments and ablations, we refer the reader to the Appendix.

\section{Discussion and Future Work}
The ELM method presents a unified approach to enforce margins in logit space, and regularise the distribution of embeddings.
Our argument for the value of such regularisation is instructive, and 
such regularisation implies better generalisation as shown in Proposition~\ref{prop:gen_bound_pull}.
Yet a key question remains elusive: 
can we improve the performance of the tail group without trading off the performance of the head?
The breakdown of per-group accuracies in Figure~\ref{fig:breakdown}, especially on the challenging iNaturalist problem, 
offers some hope of this possibility.
Exploring conditions under which this is possible  would be a worthwhile direction for future work.
More broadly, studying the efficacy of ELM in fairness settings with under-represented samples,
to ensure it does not introduce unforeseen biases,
is another important direction.

\bibliographystyle{plainnat}
\bibliography{logit_embed}

\clearpage

\appendix
\onecolumn

\begin{center}
  {\LARGE \ourtitle} \\[5mm]
  {\large Supplementary Material}
  \vspace{5mm}
\end{center}

\section{Experiment setup: hyperparameters}
\label{app:hyperparameters}

\subsection{Architecture and optimisation hyperparameters}

To facilitate a fair comparison, we use the same setup for all the methods for each dataset.
These settings are summarised in Table~\ref{tbl:hyper}.

For CIFAR,
we use the standard CIFAR data augmentation procedure used in previous works such as~\citet{Cao:2019, He:2016}, where 4 pixels are padded on each size and a random $32 \times 32$ crop is taken.
Images are horizontally flipped with a probability of 0.5.
For ImageNet and iNaturalist,
we apply the standard data augmentation comprising of random cropping and flipping as described in~\citet{Goyal:2017}.

\begin{table}[ht]
    \centering
    \renewcommand{\arraystretch}{1.25}
    \begin{tabular}{@{}lp{1.5in}ll@{}}
        \toprule
        & \textbf{CIFAR*-LT} & \textbf{ImageNet-LT} & \textbf{iNaturalist} \\
        \toprule
        Model & CIFAR ResNet-32 & ResNet-50 & ResNet-50 \\
        Optimiser          & \multicolumn{3}{c}{SGD with momentum} \\
        Base learning rate & \multicolumn{3}{c}{0.4} \\
        Epochs & 256 & 90 & 90 \\
        Batch size & 128 & 512 & 1024 \\
        Schedule & Linear warmup for the first 15 epochs, and a decay of 0.1 at the 96th, 192nd, and 224th epoch & Cosine & Cosine \\
        Weight decay & $10^{-4}$ & $5 \times 10^{-4}$ & $10^{-4}$ \\
        \bottomrule
    \end{tabular}
    \caption{Summary of hyperparameters.}
    \label{tbl:hyper}
\end{table}

\subsection{Settings for ELM}
We detail hyperparameter settings for the proposed ELM.
For all the four datasets considered in \S\ref{sec:experiments}, we set the pull margin to be  $\alpha_y \propto \Pr(y)^a$
for some $a > 0$.
This choice allows the ELM to  pull embeddings of tail
classes more strongly compared to frequent classes.
Empirically we found that setting $a= \frac{1}{2}$ or $a=1$ works well for most datasets.

\paragraph{CIFAR10-LT}:
We use $\alpha_y = \Pr(y)$, and set the
regularization parameter $\lambda = 0.01$.

\paragraph{CIFAR100-LT}:
We set $\alpha_y= \sqrt{n_y}$, and
regularization parameter $\lambda = 0.01$ (see \eqref{eqn:elm}),
where $n_y$ denotes the number of training samples in class $y$.

\paragraph{ImageNet-LT}:
We set $\alpha_y = 10 \times \sqrt{\Pr(y)}$ and $\lambda=0.001$.

\paragraph{iNaturalist}:
We set $\alpha_y = \sqrt{n_y}$, and set $\lambda = 0.01$.

Based on our investigation, the choice $\alpha_y = n_y$ with $\lambda = 0.01$ appears to offer good performance across many datasets.
See also \S\ref{sec:reg_sensitivity} where we show how various choices of $\alpha_y$ and $\lambda$ affect the test accuracy of ELM.

\clearpage

\section{Experiments: additional results}

\subsection{Additional margin distributions}

Figure~\ref{fig:margin-cdf} plots the cumulative margin distributions
($\gamma( x, y ) \defEq f_y( x ) - \max_{y' \neq y} f_{y'}( x )$ for instances $x$ with label $y$)
for various methods on CIFAR100-LT.
Here, we clearly see a significant gap between ERM and method that enforce a logit margin,
which are in turn bested with those that enforce an embedding margin.

\begin{figure}[ht]
    \centering

    \resizebox{\linewidth}{!}{
    \subcaptionbox{Label 10 (Head).}{
    \includegraphics[scale=0.28]{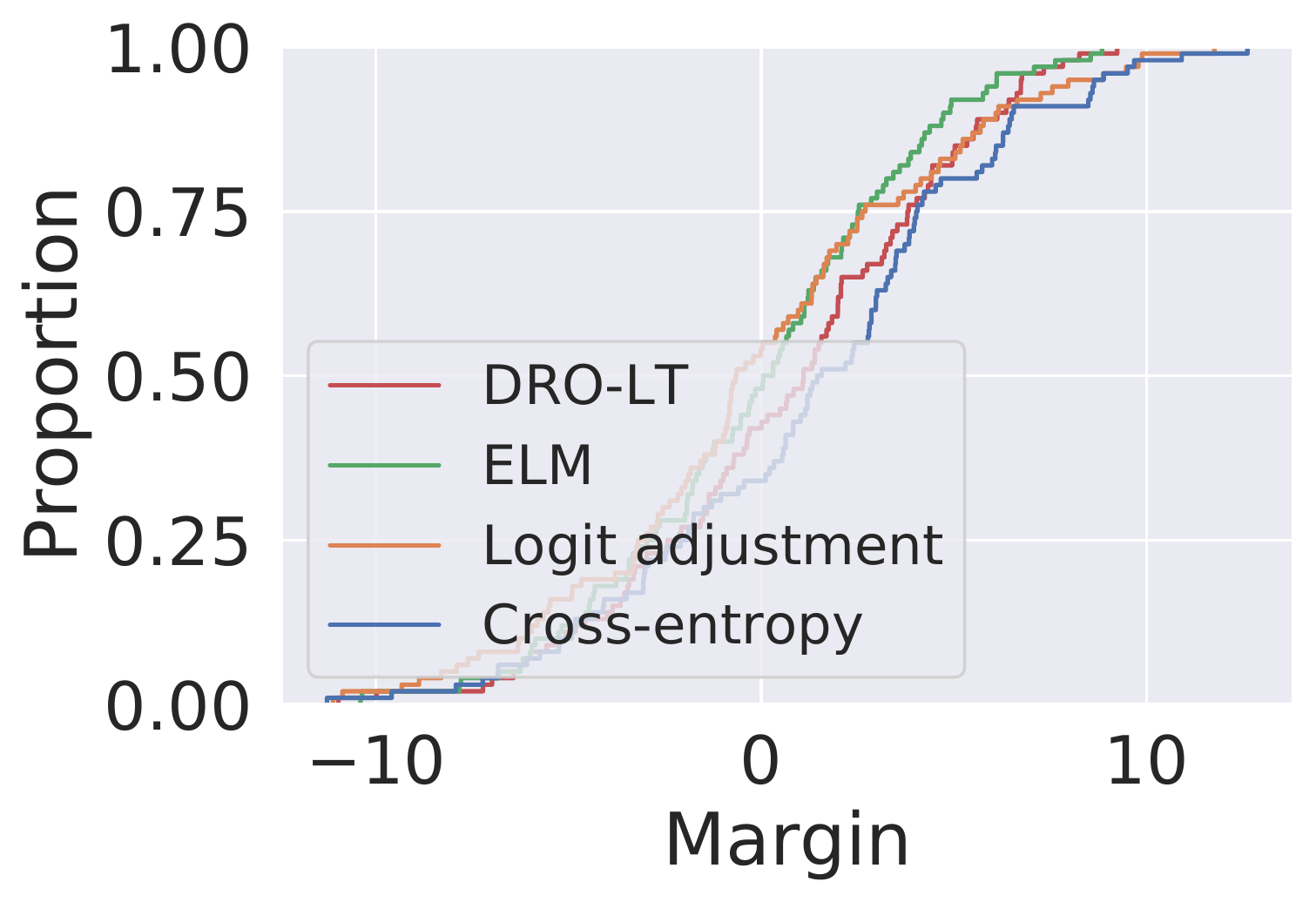}
    }
    \quad
    \subcaptionbox{Label 50 (Torso).}{
    \includegraphics[scale=0.28]{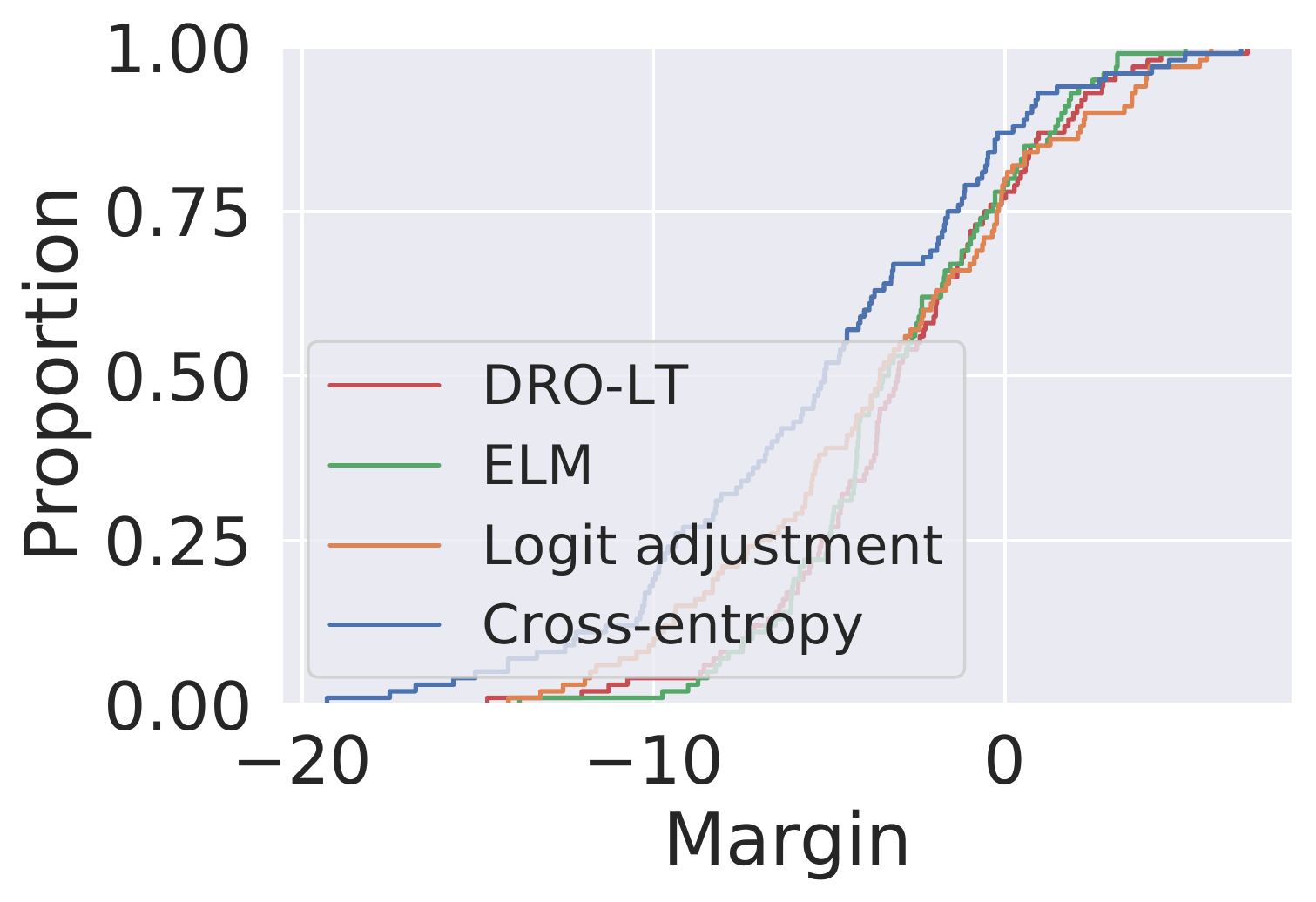}
    }
    \quad
    \subcaptionbox{Label 90 (Tail).}{
    \includegraphics[scale=0.28]{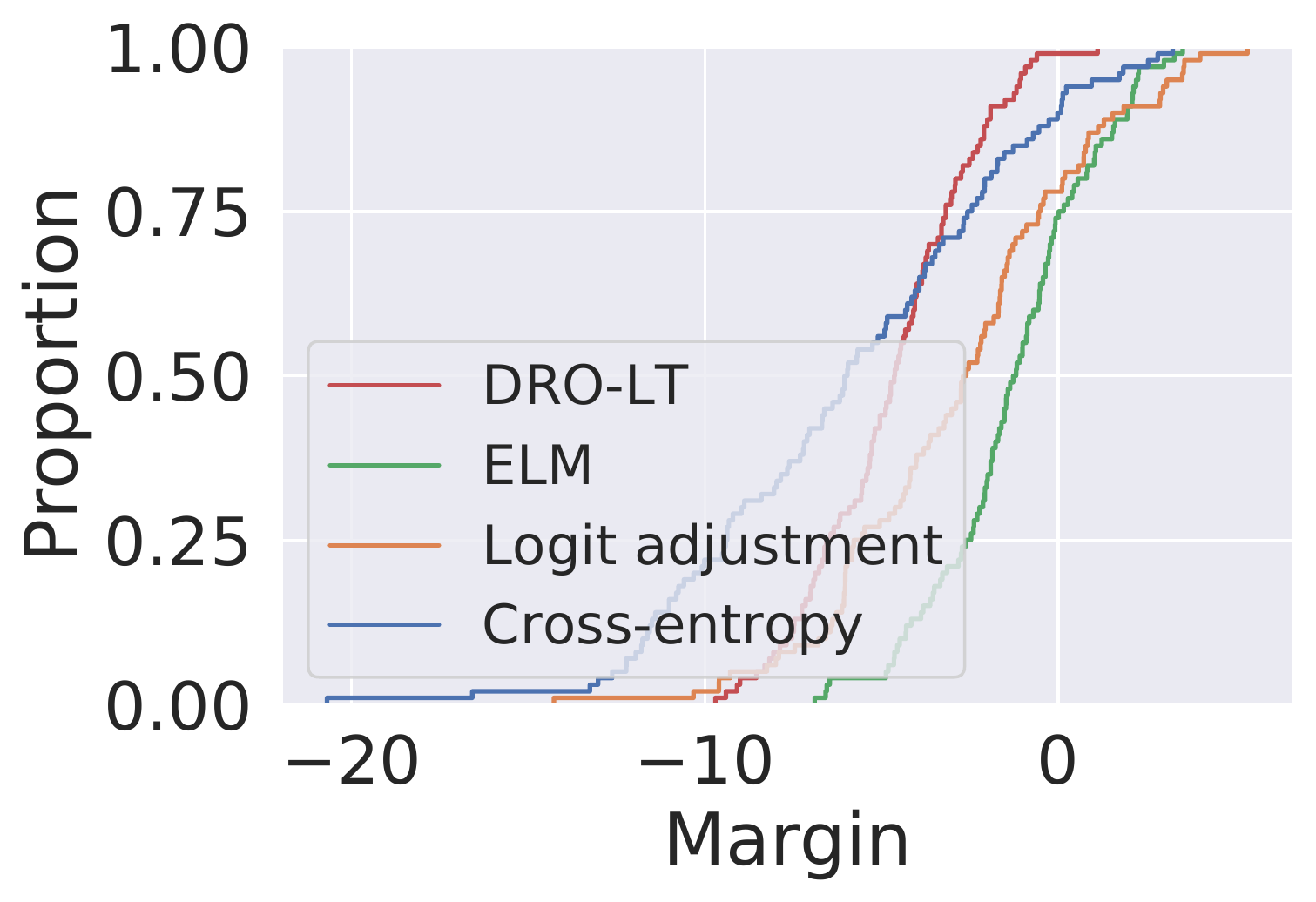}
    }
    }

    \caption{Margin CDF plots on CIFAR100-LT.
    For a given label $y$, we plot the cumulative distribution of the margins $\gamma( x, y ) \defEq f_y( x ) - \max_{y' \neq y} f_{y'}( x )$ for instances $x$ with label $y$.
    Logit adjustment is seen to shift margins favourably on the tail label,
    at some expense on head labels.
    ELM is seen to consistently reduce the variance of the margin distribution.
    }
    \label{fig:margin-cdf}
\end{figure}

\subsection{Sensitivity to Regularisation Strength}
\label{sec:reg_sensitivity}

Figure~\ref{fig:sensitivity} shows how the choice of regularisation strength $\lambda$ affects final test set performance of ELM.
When $\lambda$ is too large, performance suffers considerably;
when $\lambda$ is too small, performance is indistinguishable from that of standard cross-entropy minimisation.
However, for intermediate values of $\lambda$ we see some gains, indicating the value of the regulariser.

\begin{figure}[ht]
    \centering
    \resizebox{0.99\linewidth}{!}{
    \subcaptionbox{\LARGE CIFAR100-LT}{
    \includegraphics[scale=0.9]{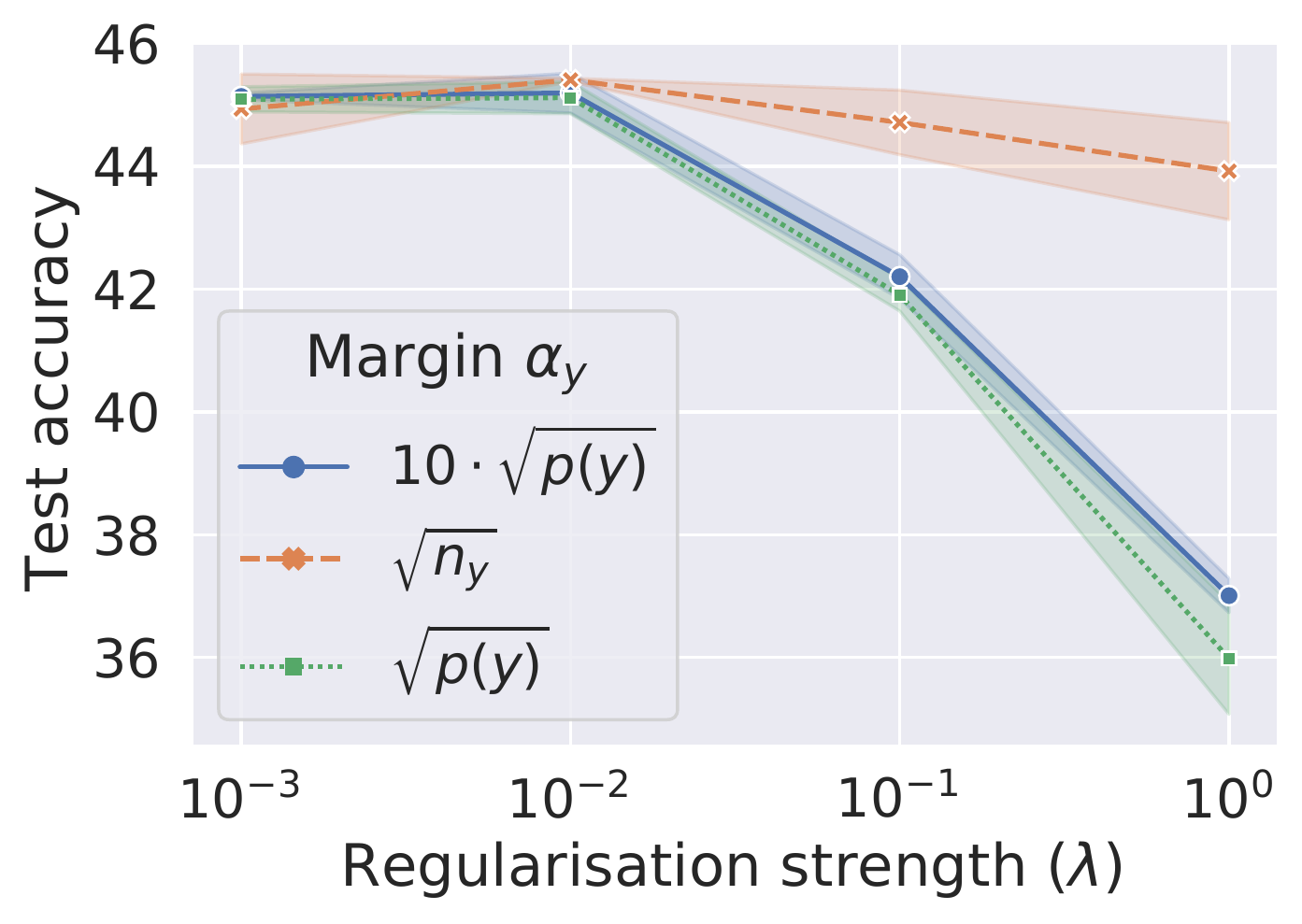}
    }
    \subcaptionbox{\LARGE ImageNet-LT}{
    \includegraphics[scale=0.9]{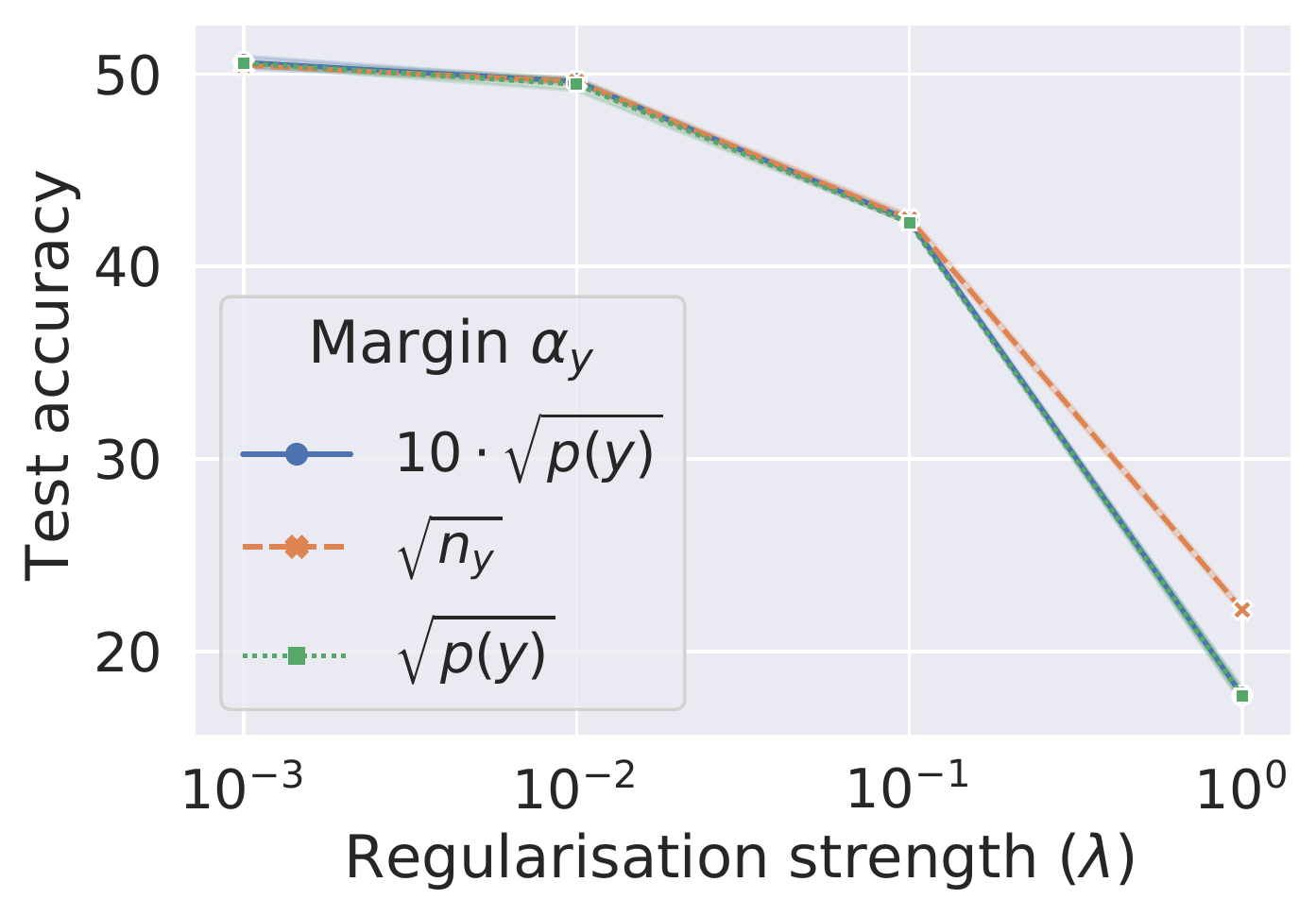}
    }
    \subcaptionbox{\LARGE iNaturalist}{
    \includegraphics[scale=0.9]{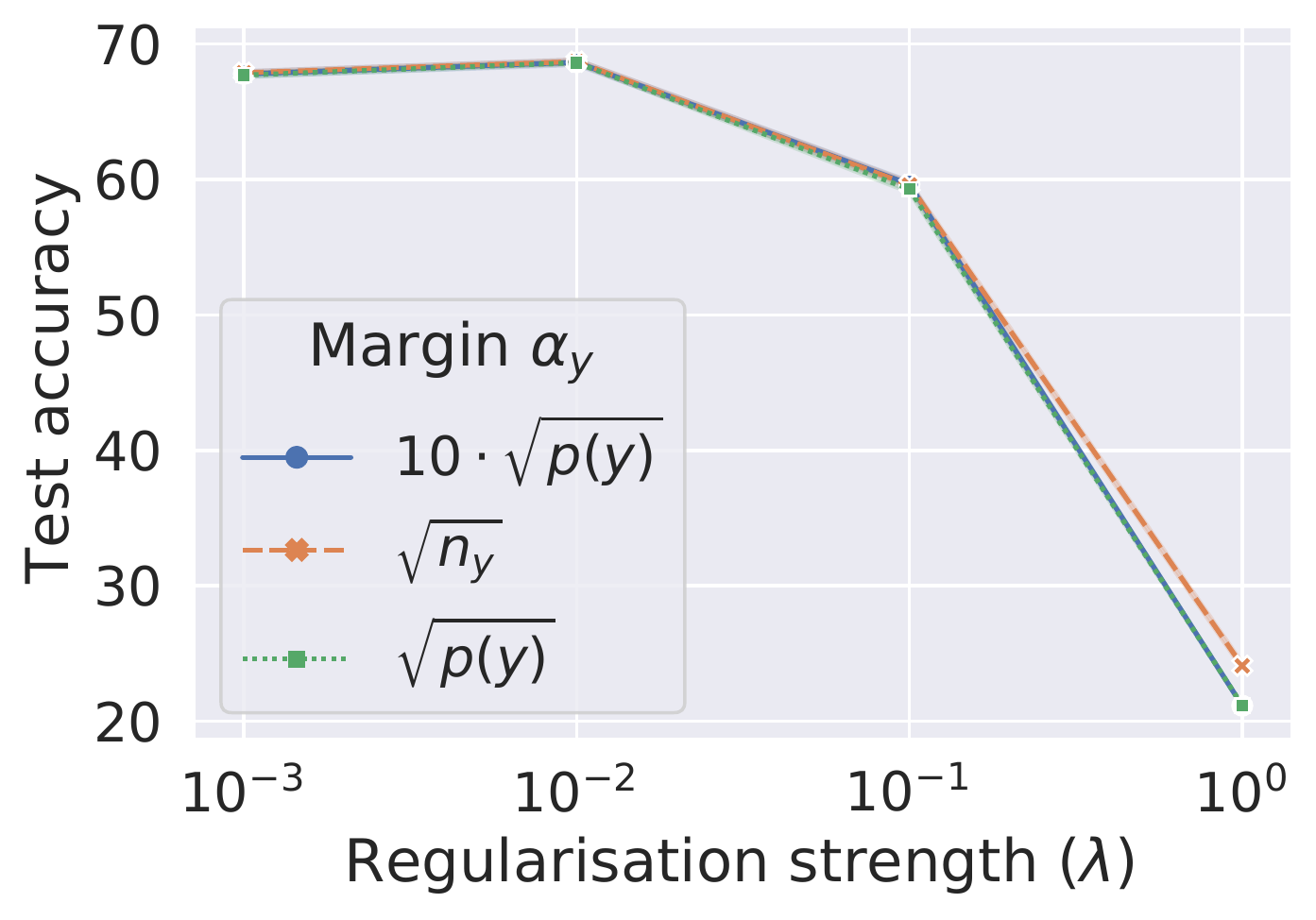}
    }
    }
    \caption{Impact of choice of $\lambda$ on final balanced accuracy.
    When $\lambda$ is too large, performance is seen to suffer;
    however, an intermediate choice yields gains over standard cross entropy
    minimisation.
    }
    \label{fig:sensitivity}
\end{figure}

\clearpage

\section{Proof}

\subsection{Proof of Proposition \ref{prop:pull_and_variance}}
\label{sec:proof_pull_and_variance}

We start with the definition of $\Omega_{\mathrm{pull}}(x,y)$:
\begin{align*}
\Omega_{\mathrm{pull}}(x,y) & =\log\text{\ensuremath{\left[1+\ensuremath{\sum_{x^{+}\in S_{y}\backslash\{x\}}}e^{\|\Phi(x)-\Phi(x^{+})\|^{2}-\alpha_{y}}\right]}}\\
 & =\log\text{\ensuremath{\left[\frac{1}{|S_{y}|-1}\sum_{x^{+}\in S_{y}\backslash\{x\}}1+\frac{1}{|S_{y}|-1}\ensuremath{\sum_{x^{+}\in S_{y}\backslash\{x\}}(|S_{y}|-1)}e^{\|\Phi(x)-\Phi(x^{+})\|^{2}-\alpha_{y}}\right]}}\\
 & \stackrel{(a)}{\ge}\frac{1}{|S_{y}|-1}\sum_{x^{+}\in S_{y}\backslash\{x\}}\log\text{\ensuremath{\left[1+(|S_{y}|-1)e^{\|\Phi(x)-\Phi(x^{+})\|^{2}-\alpha_{y}}\right]}}\\
 & \ge\frac{1}{|S_{y}|-1}\sum_{x^{+}\in S_{y}\backslash\{x\}}\log\text{\ensuremath{\left[(|S_{y}|-1)e^{\|\Phi(x)-\Phi(x^{+})\|^{2}-\alpha_{y}}\right]}}\\
 & =\frac{1}{|S_{y}|-1}\sum_{x^{+}\in S_{y}\backslash\{x\}}\|\Phi(x)-\Phi(x^{+})\|^{2}-\alpha_{y}+\log\left(|S_{y}|-1\right),
\end{align*}
where at $(a)$ we use Jensen's inequality i.e., $\log\left(\frac{1}{m}\sum_{i=1}^{m}f(x_{i})\right)\ge\frac{1}{m}\sum_{i=1}^{m}\log(f(x_{i}))$.
This implies that
\begin{align*}
\bar{\Omega}_{\mathrm{pull}}(y) & =\frac{1}{|S_{y}|}\sum_{x\in S_{y}}\Omega_{\mathrm{pull}}(x,y)\\
 & \ge\frac{1}{|S_{y}|(|S_{y}|-1)}\sum_{x\in S_{y}}\sum_{x^{+}\in S_{y}\backslash\{x\}}\|\Phi(x)-\Phi(x^{+})\|^{2}-\alpha_{y}+\log\left(|S_{y}|-1\right)\\
 & =\frac{1}{|S_{y}|(|S_{y}|-1)}\sum_{x,x^{+}\in S_{y}}\|\Phi(x)-\Phi(x^{+})\|^{2}-\alpha_{y}+\log\left(|S_{y}|-1\right),
\end{align*}
where the last line follows from the fact that $\|\Phi(x)-\Phi(x)\|^{2}=0$.

Observe that $\frac{1}{|S_{y}|^{2}}\sum_{x,x^{+}\in S_{y}}\|\Phi(x)-\Phi(x^{+})\|^{2}=\frac{2}{|S_{y}|}\sum_{x\in S_{y}}\|\Phi(x)-\hat{\mu}_{y}\|^{2}$
where $\hat{\mu}_{y}=\frac{1}{|S_{y}|}\sum_{x\in S_{y}}\Phi(x)$.
To show this, we will start by expanding the square in $\frac{2}{|S_{y}|}\sum_{x\in S_{y}}\|\Phi(x)-\hat{\mu}_{y}\|^{2}$.
In the following derivation, it is useful to note that $\frac{1}{|S_{y}|}\sum_{x\in S_{y}}1=1$
is used in many steps, and that $\frac{1}{|S_{y}|}\sum_{x\in S_{y}}1$
is introduced for the purpose of rearranging the expression into a
form that has two nested sums:

\begin{align*}
 & \frac{2}{|S_{y}|}\sum_{x\in S_{y}}\|\Phi(x)-\hat{\mu}_{y}\|^{2}\\
 & =\frac{2}{|S_{y}|}\sum_{x\in S_{y}}\left[\|\Phi(x)\|^{2}+\|\hat{\mu}_{y}\|^{2}-2\Phi(x)^{\top}\hat{\mu}_{y}\right]\\
 & =\frac{2}{|S_{y}|}\sum_{x\in S_{y}}\|\Phi(x)\|^{2}+2\thinspace\bigg\|\frac{1}{|S_{y}|}\sum_{x\in S_{y}}\Phi(x)\bigg\|^{2}-\frac{4}{|S_{y}|}\sum_{x\in S_{y}}\Phi(x)^{^{\top}}\left(\frac{1}{|S_{y}|}\sum_{x^{+}\in S_{y}}\Phi(x^{+})\right)\\
 & =\frac{2}{|S_{y}|^{2}}\sum_{x,x^{+}\in S_{y}}\|\Phi(x)\|^{2}+2\thinspace\left(\frac{1}{|S_{y}|}\sum_{x\in S_{y}}\Phi(x)\right)^{\top}\left(\frac{1}{|S_{y}|}\sum_{x'\in S_{y}}\Phi(x')\right)-\frac{4}{|S_{y}|}\sum_{x\in S_{y}}\Phi(x)^{^{\top}}\left(\frac{1}{|S_{y}|}\sum_{x^{+}\in S_{y}}\Phi(x^{+})\right)\\
 & =\frac{2}{|S_{y}|^{2}}\sum_{x,x^{+}\in S_{y}}\|\Phi(x)\|^{2}+\frac{2}{|S_{y}|^{2}}\sum_{x,x^{+}\in S_{y}}\Phi(x)^{\top}\Phi(x^{+})-\frac{4}{|S_{y}|^{2}}\sum_{x,x^{+}\in S_{y}}\Phi(x)^{^{\top}}\Phi(x^{+})\\
 & ={\frac{2}{|S_{y}|^{2}}\sum_{x,x^{+}\in S_{y}}\|\Phi(x)\|^{2}}-\frac{2}{|S_{y}|^{2}}\sum_{x,x^{+}\in S_{y}}\Phi(x)^{^{\top}}\Phi(x^{+})\\
 & ={\frac{1}{|S_{y}|^{2}}\sum_{x,x^{+}\in S_{y}}\|\Phi(x)\|^{2}+\frac{1}{|S_{y}|^{2}}\sum_{x,x^{+}\in S_{y}}\|\Phi(x^{+})\|^{2}}-\frac{2}{|S_{y}|^{2}}\sum_{x,x^{+}\in S_{y}}\Phi(x)^{^{\top}}\Phi(x^{+})\\
 & {\normalcolor =}\frac{1}{|S_{y}|^{2}}\sum_{x,x^{+}\in S_{y}}\left[\|\Phi(x)\|^{2}+\|\Phi(x^{+})\|^{2}-2\Phi(x)^{^{\top}}\Phi(x^{+})\right]\\
 & =\frac{1}{|S_{y}|^{2}}\sum_{x,x^{+}\in S_{y}}\|\Phi(x)-\Phi(x^{+})\|^{2}.
\end{align*}
It follows that
\begin{align*}
\bar{\Omega}_{\mathrm{pull}}(y) & \ge\frac{2}{(|S_{y}|-1)}\sum_{x\in S_{y}}\|\Phi(x)-\hat{\mu}_{y}\|^{2}-\alpha_{y}+\log\left(|S_{y}|-1\right)\\
 & =\frac{2|S_{y}|}{(|S_{y}|-1)}\sum_{j=1}^{K}\frac{1}{|S_{y}|}\sum_{x\in S_{y}}(\Phi(x)_{j}-\hat{\mu}_{y,j})^{2}-\alpha_{y}+\log\left(|S_{y}|-1\right)\\
 & =\frac{2|S_{y}|}{(|S_{y}|-1)}\sum_{j=1}^{K}\hat{\mathbb{V}}[\Phi_{j}(x)\mid y]-\alpha_{y}+\log\left(|S_{y}|-1\right).
\end{align*}

\subsection{Proof of Proposition \ref{prop:gen_bound_pull} (Generalisation Bound)}

Before we give proof for Proposition \ref{prop:gen_bound_pull}, we present a few
lemmas (Lemmas \ref{lem:bennett}, \ref{lem:var_log_lin}, and
\ref{lem:varlog_ub}) that will be useful later for proving the proposition.
We start with Lemma \ref{lem:bennett}, a known result that gives a probabilistic
upper bound of the population mean in terms of an empirical variance.

\begin{lem}[Bennett's inequality \citep{MauPon2009}]
\label{lem:bennett} Let $Z_{1},\ldots,Z_{n}$ be i.i.d. random variables
with values in $[0,B]$ and let $\delta>0$. Then, with probability
at least $1-\delta$ in $(Z_{1},\ldots,Z_{n})$,
\begin{align*}
\text{\ensuremath{\mathbb{E}Z}} & \le\frac{1}{n}\sum_{i=1}^{n}Z_{i}+\sqrt{\frac{2\hat{\mathbb{V}}[Z]\ln2/\delta}{n}}+\frac{7B\ln2/\delta}{3(n-1)}.
\end{align*}
\end{lem}

We will use Lemma \ref{lem:bennett} as the starting point for proving our generalisation
bound of the logit-adjusted cross entropy loss.
The logit-adjusted cross entropy loss is a special case of the log loss.
Lemma \ref{lem:var_log_lin} states that the variance of the log loss is no
larger than the variance of the linear loss.
This observation, together with Lemma \ref{lem:varlog_ub}, will provide
necessary intermediate steps in our main proof for connecting the
variance of the log loss to the proposed pull objective.

\begin{lem}
\label{lem:var_log_lin}Let $\ell_{\mathrm{log}}(y,f(x)):=\log\left(1+e^{-yf(x)}\right)$
and $\ell_{\mathrm{lin}}(y,f(x))=-yf(x).$ Then, for any $y\in\{-1,1\}$,
$\mathbb{V}_{x|y}[\ell_{\mathrm{log}}(y,f(x))]\le\mathbb{V}_{x|y}[\ell_{\mathrm{lin}}(y,f(x))]$.
\end{lem}

\begin{proof}
Let $s_{y}(z):=\log\left(1+e^{-yz}\right)$ for $y\in\{-1,1\}$. Observe
that $s'_{y}(z)=-\frac{ye^{-yz}}{1+e^{-yz}}$ so that $\sup_{z}|s'_{y}(z)|\le1$.
That is, $s_{y}$ is a $1$-Lipschitz function i.e., for any $z,z'\in\mathbb{R}$,
$|s_{y}(z)-s_{y}(z')|\le|z-z'|$. By definition of variance, for any
real-valued function $h$,
\begin{align*}
\mathbb{V}[h(z)] & =\mathbb{E}[h^{2}(z)]-\mathbb{E}^{2}[h(z)]\\
 & \le\mathbb{E}[h^{2}(z)].
\end{align*}
It follows that
\begin{align*}
\mathbb{V}_{x|y}[\ell_{\mathrm{log}}(y,f(x))] & =\mathbb{V}_{x|y}[s_{y}(f(x))]\\
 & \stackrel{(a)}{=}\mathbb{V}_{x|y}[s_{y}(f(x))-s_{y}(\mathbb{E}_{x'|y}[f(x')])]\\
 & \le\mathbb{E}_{x|y}\left[\left(s_{y}(f(x))-s_{y}(\mathbb{E}_{x'|y}[f(x')])\right)^{2}\right]\\
 & \stackrel{(b)}{\le}\mathbb{E}_{x|y}\left[\left(f(x)-\mathbb{E}_{x'|y}[f(x')]\right)^{2}\right]\\
 & =\mathbb{E}_{x|y}\left[\left(yf(x)-\mathbb{E}_{x'|y}[yf(x')]\right)^{2}\right]\\
 & =\mathbb{V}_{x|y}[\ell_{\mathrm{lin}}(y,f(x))],
\end{align*}
where at $(a)$ we note that adding a constant $s_{y}(\mathbb{E}_{x'|y}[f(x')])$
does not change the variance, and at $(b)$ we use the fact that $s_{y}$
is $1$-Lipschitz.
\end{proof}

\begin{lem}
\label{lem:varlog_ub}Consider the binary classification case where $y\in\{-1,1\}$.
Let $f(x)\stackrel{.}{=}w^{\top}\Phi(x)+b\in\mathbb{R}$ be the logit
function for class $y=1$. Define
$\Omega_{\mathrm{cen}}(y):=\mathbb{E}_{x|y}\|\Phi(x)-\mu_{y}\|^{2}$
where $\mu_{y}\stackrel{.}{=}\mathbb{E}_{x|y}\Phi(x)$. Let $\Omega_{\mathrm{cen}}\stackrel{.}{=}\frac{1}{2}\left[\Omega_{\mathrm{cen}}(-1)+\Omega_{\mathrm{cen}}(1)\right]$
be the center loss. Then, for
any $\Delta_{y}\in\mathbb{R}$, we have $\mathbb{V}_{x|y}[\ell_{\mathrm{log}}(y,f(x)+\Delta_{y})]\le\|w\|^{2}\mathrm{tr}(C_{y})$ where $C_{y}\stackrel{.}{=}\mathbb{E}_{x|y}(\Phi(x)-\mu_{y})(\Phi(x)-\mu_{y})^{\top}.$
\end{lem}

\begin{proof}
Let $\lambda_{\max}(A)$ be the the maximum eigenvalue of a square
positive definite matrix $A$. Consider the variance of the linear
loss as in Lemma \ref{lem:var_log_lin}. Observe that
\begin{align}
\mathbb{V}_{x|y}[\ell_{\mathrm{lin}}(y,f(x)+\Delta_{y})] & =\mathbb{V}_{x|y}[-y(f(x)+\Delta_{y})]\nonumber \\
 & =\mathbb{V}_{x|y}[f(x)]\nonumber \\
 & =w^{\top}C_{y}w\nonumber \\
 & \le\|w\|^{2}\|C_{y}\|_{2}\nonumber \\
 & =\|w\|^{2}\lambda_{\mathrm{max}}(C_{y})\nonumber \\
 & \le\|w\|^{2}\mathrm{tr}(C_{y}).\label{eq:tr_cy_bound}
\end{align}
\end{proof}

\paragraph{Proof of Proposition \ref{prop:gen_bound_pull}}
We are now ready to prove Proposition \ref{prop:gen_bound_pull}.

\begin{proof}
We first consider the class-conditional logit adjusted loss $\mathbb{E}_{x|y}\ell_{\mathrm{log}}(y,f(x)+\Delta_{y})$.
For $y\in\{-1,1\}$, Lemma \ref{lem:bennett} implies that with probability
at least $1-\delta/2$
\begin{align*}
\mathbb{E}_{x|y}\ell_{\mathrm{log}}(y,f(x)+\Delta_{y}) & \le\frac{1}{|S_{y}|}\sum_{x\in S_{y}}\ell_{\mathrm{log}}(y,f(x)+\Delta_{y})+\sqrt{\frac{2\hat{\mathbb{V}}_{x|y}[\ell_{\mathrm{log}}(y,f(x)+\Delta_{y})]\ln4/\delta}{|S_{y}|}}+\frac{7B\ln4/\delta}{3(|S_{y}|-1)}.
\end{align*}
As a consequence of the union bound, we have with probability at least
$1-\delta$,
\begin{align*}
 & \sum_{y\in\{-1,1\}}\mathbb{P}(y)\mathbb{E}_{x|y}\ell_{\mathrm{log}}(y,f(x)+\Delta_{y})\\
 & \le\sum_{y\in\{-1,1\}}\mathbb{P}(y)\left[\frac{1}{|S_{y}|}\sum_{x\in S_{y}}\ell_{\mathrm{log}}(y,f(x)+\Delta_{y})+\sqrt{\frac{2\hat{\mathbb{V}}_{x|y}[\ell_{\mathrm{log}}(y,f(x)+\Delta_{y})]\ln2/\delta}{|S_{y}|}}+\frac{7B\ln2/\delta}{3(|S_{y}|-1)}\right]\\
 & =\sum_{y\in\{-1,1\}}\frac{\mathbb{P}(y)}{|S_{y}|}\sum_{x\in S_{y}}\ell_{\mathrm{log}}(y,f(x)+\Delta_{y})\\
 & \phantom{=}+\sum_{y\in\{-1,1\}}\mathbb{P}(y)\sqrt{\frac{2\hat{\mathbb{V}}_{x|y}[\ell_{\mathrm{log}}(y,f(x)+\Delta_{y})]\ln2/\delta}{|S_{y}|}}+\frac{7B\ln2/\delta}{3}\sum_{y\in\{-1,1\}}\frac{\mathbb{P}(y)}{|S_{y}|-1}\\
 & \stackrel{.}{=}T_{n}+V_{n}+\frac{7B\ln2/\delta}{3}\sum_{y\in\{-1,1\}}\frac{\mathbb{P}(y)}{|S_{y}|-1}\\
 & \stackrel{.}{=}(\heartsuit),
\end{align*}
where we define $T_{n}\stackrel{.}{=}\sum_{y\in\{-1,1\}}\frac{\mathbb{P}(y)}{|S_{y}|}\sum_{x\in S_{y}}\ell_{\mathrm{log}}(y,f(x)+\Delta_{y})$
and $V_{n}\stackrel{.}{=}\sum_{y\in\{-1,1\}}\mathbb{P}(y)\sqrt{\frac{2\hat{\mathbb{V}}_{x|y}[\ell_{\mathrm{log}}(y,f(x)+\Delta_{y})]\ln2/\delta}{|S_{y}|}}$.
We upper bound $V_{n}$ as
\begin{align*}
V_{n} & \stackrel{(a)}{\le}\sqrt{\sum_{y\in\{-1,1\}}\frac{2\mathbb{P}(y)}{|S_{y}|}\hat{\mathbb{V}}_{x|y}[\ell_{\mathrm{log}}(y,f(x)+\Delta_{y})]\ln2/\delta}\\
 & \stackrel{(b)}{\le}\sqrt{\ln\frac{2}{\delta}\sum_{y\in\{-1,1\}}\frac{2\mathbb{P}(y)}{|S_{y}|}\|w\|^{2}\mathrm{tr}(\hat{C}_{y})}\\
 & \stackrel{.}{=}(\bigstar)
\end{align*}
where at $(a)$ we use Jensen's inequality, $(b)$ follows from Lemma
\ref{lem:varlog_ub}, $\hat{C}_{y}\stackrel{.}{=}\frac{1}{|S_{y}|}\sum_{x\in S_{y}}(\Phi(x)-\hat{\mu}_{y})(\Phi(x)-\hat{\mu}_{y})^{\top},$
and $\hat{\mu}_{y}\stackrel{.}{=}\frac{1}{|S_{y}|}\sum_{x\in S_{y}}\Phi(x)$.
By Proposition \ref{prop:pull_and_variance}, we have
\begin{align*}
\frac{2|S_{y}|}{|S_{y}|-1}\cdot\mathrm{tr}(\hat{C}_{y})-\alpha_{y}+\log\left(|S_{y}|-1\right) & \le\bar{\Omega}_{\mathrm{pull}}(y),\text{ implying that}\\
2\mathrm{tr}(\hat{C}_{y}) & \le\bar{\Omega}_{\mathrm{pull}}(y)+\alpha_{y}-\log\left(|S_{y}|-1\right)\\
 & \le\bar{\Omega}_{\mathrm{pull}}(y)+\alpha_{y}.
\end{align*}
The last inequality suggests that
\begin{align*}
V_{n}\le(\bigstar) & \le\|w\|\sqrt{\ln\frac{2}{\delta}\sum_{y\in\{-1,1\}}\frac{\mathbb{P}(y)}{|S_{y}|}\left[\bar{\Omega}_{\mathrm{pull}}(y)+\alpha_{y}\right]}.
\end{align*}
Combining the last line with $(\heartsuit)$ gives the result.
\end{proof}

\section{Additional related work}
\label{app:related}

\textbf{Improved embedding geometry}.
Several works have studied means of improving the geometry of learned embeddings for clasification tasks.
The center loss~\citep{Wen:2016} pulls sample embeddings towards their class
centroid $\mu_y$:
$$ \Omega_{\rm center}( x, y ) \defEq \| \Phi( x ) - \mu_y \|_2^2. $$
Conversely, several works have studied regularisers that push apart embeddings~\citep{Zhang:2017b,Hayat:2019,Krichene:2018,Yang:2020};
e.g.,
under the assumption that embeddings are normalised,
the spreadout regulariser~\citep{Zhang:2017b} is
\begin{align}
M_1 = \frac{1}{|\{x \neq x'\}|}\sum_{x \neq x'}\Phi( x )^\top \Phi( x' ) \\
M_2 = \frac{1}{|\{x \neq x'\}|} \sum_{x' \neq x} ( \Phi( x )^\top \Phi( x' ) )^2 \\
\Omega_{\rm spread}( x, y ) \defEq M_1^2 + \max \left(0, M_2 - \frac{1}{d}\right).
\end{align}

Similar regularisers were also explored in a long-tailed setting by~\citet{Zhong:2019}.
None of the above techniques consider an explicit margin for their ``pull'' or ``push'' regularisation.

\textbf{Deep LDA}.
Fisher linear discriminant analysis~\citep{Fisher:1936} is a classical means of tackling classification,
which relies on finding projections that minimise intra-class variance (i.e., pull together projected scores) and maximise inter-class variance (i.e., push apart projected scores).
In a deep learning context,~\citet{DorKelWid2015} proposed a form of deep LDA.
This is not attuned to the long-tail setting, and does not enforce classification or embedding margins.

\clearpage

\end{document}